\documentclass{article}

\usepackage{microtype}
\usepackage{graphicx}
\usepackage{subfigure}
\usepackage{subcaption}
\usepackage{booktabs} 
\usepackage{algorithm}
\usepackage{algorithmic}

\usepackage[hidelinks,colorlinks=true,linkcolor=blue,citecolor=blue]{hyperref}

\usepackage{amsmath}
\usepackage{amssymb}
\usepackage{mathtools}
\usepackage{amsthm}
\usepackage{natbib}
\usepackage{color}
\usepackage{amsfonts}       
\usepackage{url}
\usepackage{wrapfig}
\usepackage{tikz}
\usepackage{footmisc}
\usepackage{enumerate}
\usepackage{bibentry}
\usepackage{thmtools, thm-restate}
\nobibliography*

\usepackage{physics}
\mathtoolsset{showonlyrefs}
\usepackage{etoolbox}
\usepackage{makecell}
\usepackage{enumerate}
\usepackage{enumitem}
\usepackage{pifont}
\usepackage{soul}
\usepackage{ulem}
\usepackage{cancel}
\theoremstyle{plain}
\newtheorem{theorem}{Theorem}
\newtheorem{proposition}{Proposition}
\newtheorem{lemma}{Lemma}
\newtheorem{corollary}{Corollary}
\theoremstyle{definition}
\newtheorem{definition}{Definition}

\theoremstyle{plain}
\newcommand{\fS}{\mathcal{S}}
\newcommand{\fA}{\mathcal{A}}

\newcommand{\fT }{\mathcal{T}}

\newcommand{\fF}{\mathcal{F}}

\newcommand{\fL}{\mathcal{L}}
\newcommand{\cE}{\mathcal{E}}
\newcommand{\R}[1][]{\mathbb{R}^{#1}}
\newcommand{\E}{\mathbb{E}}
\newcommand{\ns}{{\abs{\fS}}}

\newcommand{\tref}[1]{\text{\ref{#1}}}

\newcommand{\1}{\mathbf{1}}
\newcommand{\0}{\mathbf{0}}

\newcommand{\pmdp}{p_{\text{MDP}}}
\newcommand{\rmdp}{r_{\text{MDP}}}

\newcommand{\LinAttn}{\text{LinAttn}}

\allowdisplaybreaks

\usepackage[textsize=tiny]{todonotes}

\usepackage[preprint, nonatbib]{neurips_2026}

\title{Convergence and Emergence of
In-Context Reinforcement Learning with Chain of Thought}

\author{%
  Zixuan Xie \\
  University of Virginia\\
  \texttt{xie.zixuan@email.virginia.edu} \\
  \And
  Xinyu Liu \\
  University of Virginia\\
  \texttt{xinyuliu@virginia.edu} \\
  \And
  Rohan Chandra \\
  University of Virginia\\
  \texttt{rohanchandra@virginia.edu} \\
  \And
  Shangtong Zhang \\
  University of Virginia\\
  \texttt{shangtong@virginia.edu} \\
}

\begin{document}

\maketitle

\begin{abstract}
In-context reinforcement learning (ICRL) refers to the ability of RL agents to adapt to new tasks at inference time without parameter updates by conditioning on additional context.
Recent empirical studies further demonstrate that Chain-of-Thought (CoT) generation can amplify this ICRL capability.
This paper is the first to provide a theoretical understanding on how CoT interacts with ICRL. 
We conduct our analysis in a policy evaluation setup with linear Transformer.
We prove that with specific Transformer parameters,
the CoT generation process is equivalent to repeatedly executing temporal difference learning updates.
Additionally, we provide finite sample convergence analysis showing that the policy evaluation error 
decreases geometrically with CoT length and eventually saturates at a statistical floor determined by the context length.
We also prove that the desired Transformer parameters are a global minimizer of the pretraining loss,
providing a theoretical understanding on the empirical emergence of those parameters. 
\end{abstract}

\section{Introduction}
\label{sec:intro}
Reinforcement learning (RL, \citet{sutton2018reinforcement}) provides a foundational framework where agents learn sequential decision-making through interaction. 
A standard paradigm in deep RL consists of training an agent with an RL algorithm on a target problem and deploying the fixed policy within that same domain \citep{mnih2015human,schulman2015trust,schulman2017proximal}. However, this workflow lacks flexibility against task variations. If the underlying goals or environment dynamics change, the specialized policy often fails, requiring the agent to be retrained from scratch to adapt, which can be computationally expensive.
In-Context Reinforcement Learning (ICRL, \citet{laskin2023context,moeini2025survey}) mitigates this limitation by enabling adaptation at inference time without parameter updates. In this framework, an agent is pretrained on a diverse distribution of tasks, typically formulated as either policy evaluation or control problems \citep{moeini2025survey}. Policy evaluation involves estimating value functions, whereas control focuses on searching for optimal policies. 
At inference time, the network parameters are frozen and the agent conditions on additional context generated within the new task.
A standard form of such context is the trajectory of recent interactions, comprising past observations, actions, and rewards \citep{laskin2023context}. Empirical evidence \citep{laskin2023context, team2023human, elawady2024relic, raparthy2024generalization, grigsby2024amago, grigsby2024amago2,liu2025locoformer, Beaussant2025scaling} shows that performance often improves as the context grows, even on test tasks that deviate significantly from the pretraining distribution. 
Thus, this improvement is difficult to explain by memorization of the pretraining tasks.
It instead suggests an implicit implementation of an internal RL process within the network's forward pass that leverages the context for adaptation \citep{moeini2025survey}. 
We refer to this capability of reinforcement learning at inference time as ICRL.

This ICRL capability has been shown to arise under a diverse array of pretraining algorithms, which can be broadly divided into supervised pretraining and reinforcement pretraining.
In supervised pretraining \citep{laskin2023context, shi2023cross,kirsch2023towards,lee2024supervised, zisman2024emergence,dai2024incontext,lin2023transformers}, the network is trained by imitation to match the behavior of a known RL algorithm on curated learning trajectories, so it is less surprising if the forward pass resembles that algorithm.
A more intriguing phenomenon arises from reinforcement pretraining, where the agent network can be a sequence model such as a Transformer \citep{vaswani2017attention} and the pretraining algorithm is simply a variant of standard deep RL algorithm
\citep{duan2016rl, wang2016learning, grigsby2024amago}.
Here, the network is trained to output good actions for control tasks or to make accurate value predictions for policy evaluation tasks, yet the resulting forward pass can still behave like an RL algorithm.
This ICRL capability is also observed in large language models (LLMs).
For example, 
\citet{song2025reward} 
report that multi-round prompting with scalar rewards can yield consistent inference-time improvement for LLMs in control tasks.
Particularly, \citet{song2025reward} rely on Chain-of-Thought (CoT, \citet{wei2022chain}) generation to amplify LLM's ICRL capability.

Many works have been done to theoretically understand ICRL 
with most focusing on supervised pretraining.
See Section~\tref{sec:related} for a more detailed exposition.
The volume of works to theoretically understand ICRL with reinforcement pretraining is much thinner.
Notably works include \citet{wang2025ictd,wang2025emergence} that study ICRL under a policy evaluation setup.
Specifically,
\citet{wang2025ictd} prove that with certain parameters,
the layer-by-layer forward pass of linear Transformers is equivalent to step-by-step update of temporal difference (TD) learning \citep{sutton1988learning},
arguably the most powerful algorithm for policy evaluation.
\citet{wang2025emergence} further prove that those parameters are a global minimizer of some reinforcement pretraining loss.
However,
\citet{wang2025ictd,wang2025emergence} do not explain the role of CoT in ICRL,
a gap that we shall close in this paper.

Specifically, we make three contributions toward closing this gap. First, we constructively prove that under specific parameters, stepwise CoT generation in a linear Transformer is equivalent to iterating TD updates on the context. Second, we establish inference-time convergence guarantees with context constructed from both known and unknown dynamics, showing that the policy evaluation error decreases geometrically up to a statistical floor controlled by the context length. Third, we prove that the desired parameters above is a global minimizer for a TD-based pretraining algorithm and we empirically validate the emergence of those parameters in controlled environments.

\section{Background} \label{sec:background}

\paragraph{Notation.}
All vectors are treated as column vectors.
We denote the identity matrix in $\R[n \times n]$ by $I_n$.
We use $0_{m \times n}$ and $1_{m \times n}$ to denote the zero matrix and the all-ones matrix.
Besides, $\0$ and $\1$ are used to denote the zero vector and the all-ones vector respectively, where the dimension is clear from context.
Let $e_i$ be the standard basis vector for the $i$-th coordinate.
The transpose of a matrix $Z$ is denoted by $Z^\top$.
The inner product of two vectors $x$ and $y$ is denoted by $\langle x, y \rangle \doteq x^\top y$.
For a symmetric positive semidefinite matrix $M$, define $\|x\|_M^2 \doteq x^\top M x$.
We use $\|\cdot\|$ to denote the Euclidean norm and $\|A\|_2$ the spectral norm.
For symmetric $M$, $\lambda_{\min}(M)$ and $\lambda_{\max}(M)$ denote its extreme eigenvalues.
We use $\tilde O(\cdot)$ to suppress universal constants and polylogarithmic factors in the problem parameters.
For $N\in\mathbb{N}$, we define $[N]\doteq\{0,1,\ldots,N-1\}$.

\subsection{Reinforcement Learning and Policy Evaluation} \label{sec:rl}

We consider an infinite-horizon Markov Decision Process (MDP, \citet{bellman1957markovian}) defined by a tuple
$(\fS,\fA,\pmdp,\rmdp,\gamma, p_0)$.
Here $\fS$ is a finite state space, $\fA$ is a finite action space,
$\pmdp(\cdot | s,a)$ is the transition kernel, $\rmdp(s,a)$ is the reward function,
$\gamma \in [0,1)$ is the discount factor,
and $p_0$ is the initial distribution.
In this paper, we focus on the policy evaluation problem, where the goal is to estimate the value function of an arbitrary policy $\pi: \fS \times \fA \to [0, 1]$.
An initial state $S_0$ is sampled from $p_0$.
At timestep $t$, given the current state $S_t$, the agent samples an action $A_t$ from its policy $\pi$.
Consequently, it obtains a reward $R_{t+1}$ determined by $\rmdp(S_t,A_t)$
and moves to the next state $S_{t+1}$ according to the transition dynamics $\pmdp(\cdot | S_t,A_t)$.
A fixed policy $\pi$ reduces the MDP to a Markov Reward Process (MRP) with transition kernel
$p(s' | s) \doteq \sum_{a \in \fA} \pi(a | s)\pmdp(s' | s,a)$
and reward function
$r(s) \doteq \sum_{a \in \fA} \pi(a | s)\rmdp(s,a).$
In this case, transitions and rewards are determined solely by the current state:
$S_{t+1} \sim p(\cdot | S_t)$ and $R_{t+1} \doteq r(S_t)$.
We focus on policy evaluation under the fixed policy $\pi$, so it suffices to work with the MRP
defined by $(p_0,p,r)$ and a trajectory $(S_0,R_1,S_1,R_2,\dots)$ sampled from it.
We use $P_\pi \in \mathbb{R}^{|\fS| \times |\fS|}$ to denote the state transition matrix induced by the policy $\pi$, i.e., $P_\pi\qty[s,s']=p(s'|s)$.
Assume the Markov chain induced by $P_\pi$ on $\fS$ is irreducible and aperiodic, so it admits a unique stationary distribution $d_\pi \in \mathbb{R}^{|\fS|}$.
We use $D_\pi$ to denote the diagonal matrix whose diagonal is $d_\pi$.
The value function is
$v_\pi(s)
\doteq
\E\qty[\sum_{k=0}^{\infty} \gamma^k R_{t+k+1} | S_t=s].$
It is the unique fixed point of the Bellman expectation operator $\fT^\pi$ defined by
$(\fT^\pi v)(s)
\doteq
\E\qty[R_{t+1} + \gamma v(S_{t+1}) | S_t=s].$

\subsection{Linear TD and the MSPBE Objective} \label{sec:td}

To estimate $v_\pi$ in large state spaces, we employ linear value function approximation.
We use a feature mapping $x:\fS \to \R[d]$ and a weight vector $w \in \R[d]$.
We approximate the value at state $s$ as $v_w(s) \doteq x(s)^\top w$.
Let $X \in \R[\ns \times d]$ denote the feature matrix whose $s$-th row is $x(s)^\top$.
The vector of approximate state values across all states is written as $Xw \in \R[\ns]$.
Our objective is to find $w$ such that $Xw$ closely approximates $v_\pi$.

A widely used method to learn $w$ is Temporal Difference learning (TD, \citet{sutton1988learning}), which updates $w$ iteratively. Using $x_t \doteq x(S_t)$ for brevity, the update rule is
\begin{equation} \label{eq:td0}
w_{t+1} = w_t + \alpha_t \qty(R_{t+1} + \gamma x_{t+1}^\top w_t - x_t^\top w_t) x_t,
\end{equation}
where $\{\alpha_t\}$ is a sequence of learning rates.
Let
\begin{equation}
\label{eq:abc}
A \doteq X^\top D_\pi (I-\gamma P_\pi)X, \,
b \doteq X^\top D_\pi r,\,
C \doteq X^\top D_\pi X,
\end{equation}
where $r\in\R[\ns]$ is the reward vector whose $s$-th entry is $r(s)$.
We assume $X$ has full column rank.
\citet{tsitsiklis1997analysis} proves that $\qty{w_t}$ converge to the TD fixed point 
\begin{equation}
\label{eq:wstar}
    w_* = A^{-1}b
\end{equation}
almost surely.
Since linear approximation is restricted to the feature subspace $\mathrm{span}(X)$, we measure convergence to $w_*$ via the mean-squared projected Bellman error (MSPBE)
\begin{equation}\label{eq:mspbe}
\fL(w)
\doteq
\|\Pi\fT^\pi (Xw) - Xw\|_{D_\pi}^2,
\end{equation}
where
$\Pi \doteq X (X^\top D_\pi X)^{-1} X^\top D_\pi$ is the orthogonal projection onto $\mathrm{span}(X)$.
For linear function approximation, MSPBE admits the equivalent quadratic form
\begin{equation}\label{eq:lnmspbe}
    \fL(w) = \norm{C^{-\frac12}(Aw-b)}_2^2.
\end{equation}
In particular, $\fL(w_*)=0$ and $\fL(w)=0$ if and only if $Aw=b$.

\subsection{Transformers and Linear Self-Attention} \label{sec:trans}

We study linear self-attention, where the softmax nonlinearity in standard self-attention \citep{vaswani2017attention} is replaced by the identity map.
This form is widely used in efficient Transformer architectures and is a common proxy in theoretical analyses of in-context computation
\citep{katharopoulos2020transformers,choromanski2022rethinking,vonoswald2023transformers,ahn2024transformers,wu2024pretraining,gatmiry2024looped,zhang2023trained,huang2025transformers,wang2025ictd}.
Given a prompt matrix $Z \in \R[(3d+1)\times(n+1)]$, we define linear self-attention as
\begin{equation} \label{eq:linattn_def}
\LinAttn(Z;P,Q)
\doteq
P Z (Z^\top Q Z),
\end{equation}
where $P\in\mathbb{R}^{(3d+1)\times(3d+1)}$ consolidates the value and projection matrices, and $Q\in\mathbb{R}^{(3d+1)\times(3d+1)}$ is the merge of regular key and query matrices. 
We consider a residual linear-attention layer
\begin{equation} \label{eq:lsa-layer}
\fF(Z;P,Q) \doteq Z + \frac{1}{n}\LinAttn(Z;P,Q),
\end{equation}
where $\frac{1}{n}$ is a normalization that aligns the attention update with batch updates that average over $n$ samples.
The behavior of this update is fully governed by the matrices $P$ and $Q$, which serve as the learnable parameters in our analysis. 

\section{Transformers with CoT Can Implement Iterated TD Updates}
\label{sec:algo}

Drawing inspiration from \citet{huang2025transformers}, 
we cast policy evaluation as an \textbf{in-context weight learning task}.
Given a finite trajectory $\tau_n=(S_0,R_1,\dots,S_n)$, 
we require the Transformer to approach the TD fixed point $w_*$ defined in \eqref{eq:wstar}.
Completing this task purely in-context with a shallow network is hard, as it requires performing an iterative computation over $\tau_n$.
Chain-of-Thought (CoT) generation is a powerful way to improve the power of such a shallow network.
CoT is commonly understood as producing intermediate computations before outputting the final prediction.
In our setting, this corresponds to generating a sequence of iterates $\{w_k\}$, where each generation conditions on the growing prompt.

In this section, we claim that a single linear Transformer layer with CoT can solve the in-context weight learning task by simulating iterated batch TD updates.
Specifically, following \citet{sutton2018reinforcement}, batch TD updates $w$ using the sample average over $\tau_n$ as
\begin{equation}\label{eq:batch-td0}
w_{k+1} = w_k 
+ \frac{\alpha}{n}\sum_{j=0}^{n-1}\qty(R_{j+1}+\gamma x_{j+1}^\top w_k - x_j^\top w_k)x_j,
\end{equation}
where $\alpha>0$ is a constant stepsize.
Batch TD can be viewed as an offline variant of TD defined in \eqref{eq:td0}.
It repeatedly reuses $\tau_n$ to drive the iterates toward $w_*$, with the improvement saturating at a statistical floor, which we quantify later.

To support the claim above, 
we use $\tau_n$ as the context and organize it into a prompt matrix $Z_0\in\R[(3d+1)\times(n+1)]$ as
\begin{equation} \label{eq:z0}
Z_0 \doteq 
\begin{bmatrix}
x_0 & \cdots & x_{n-1} & \0\\
\gamma x_1 & \cdots & \gamma x_n & \0 \\
R_{1} & \cdots & R_{n} & 1\\
\0 & \cdots & \0 & w_0
\end{bmatrix}
\doteq
\begin{bmatrix}
    X_\tau& \0\\ \gamma X_\tau' & \0\\ R& 1 \\ \0& w_0
\end{bmatrix}
,
\end{equation}
where the first $n$ columns store the fixed context, and the last column initializes the weight component $w_0=\mathbf 0$.
This block-structured prompt separates the context columns from the iterate column, following \citet{huang2025transformers}.
We now formalize CoT generation under the prompt construction \eqref{eq:z0}.
Specifically, given the current generated prompt $Z_k$ at the $k$-th step, we take the last column of the output matrix $\mathcal{F}(Z_k; P, Q)_{[:,-1]} \in \mathbb{R}^{3d+1}$ as the next iterate column and append it to form
\begin{equation}
\label{eq:zk-rec}
Z_{k+1}\doteq \begin{bmatrix}
    Z_k & \fF(Z_k;P,Q)_{[:,-1]}
\end{bmatrix}.
\end{equation}
Therefore, at step $k$, the prompt has the form $Z_k \in \mathbb{R}^{(3d+1) \times (n+k+1)}$ with iterates $(w_0,\dots,w_k)$ in the buffer
\begin{equation} \label{eq:zt_def}
Z_k \doteq 
\begin{bmatrix}
    X_\tau & \mathbf{0} & \mathbf{0} & \dots & \mathbf{0}\\ 
    \gamma X_\tau' & \mathbf{0} & \mathbf{0} & \dots & \mathbf{0}\\ 
    R & 1 & 1 & \dots & 1\\ 
    \mathbf{0} & w_0 & w_1 & \dots & w_k
\end{bmatrix}.
\end{equation}

Surprisingly, with $(P,Q)$ constructed as
\begin{equation}
\label{eq:para_construct}
    P \doteq 
    \begin{bmatrix}
        0_{(2d+1) \times d} & 0_{(2d+1) \times (2d+1)} \\
        \alpha I_d & 0_{d \times (2d+1)}
    \end{bmatrix},
    \quad Q \doteq 
    \begin{bmatrix}
        0_{d \times 2d} & 0_{d \times 1} & -I_d \\
        0_{d \times 2d} & 0_{d \times 1} & I_d \\
        0_{1 \times 2d} & 1 & 0_{1 \times d} \\
        0_{d \times 2d} & 0_{d \times 1} & 0_{d \times d}
    \end{bmatrix},
\end{equation}
the last-token output $\fF(Z_k;P,Q)_{[:,-1]}$ exactly produces the batch TD iterate $w_{k+1}$ from $w_k$ in \eqref{eq:batch-td0} for every structured prompt $Z_k$ in \eqref{eq:zt_def}.
This is formalized in the following theorem.
\begin{theorem}\label{thm:onestep}
Consider the linear attention layer \eqref{eq:lsa-layer} with $(P,Q)$ defined in \eqref{eq:para_construct}.
For any sequence $\{Z_k\}$ generated by~\eqref{eq:zk-rec} with $Z_0$ defined in~\eqref{eq:z0}, the weight component $w_k$ in $Z_k$, as in~\eqref{eq:zt_def}, coincides with the $k$-th iterate of the batch TD update~\eqref{eq:batch-td0}.
\end{theorem}
The proof is in Section~\tref{proof:onestep}.
Theorem~\tref{thm:onestep} shows that under our structured prompt, a single linear Transformer layer can implement one batch TD update in its forward pass.
Combined with the CoT recursion \eqref{eq:zk-rec}, 
$k$ steps of CoT generation repeatedly apply this update on the fixed context $\tau_n$, yielding iterates $\{w_k\}$, as claimed above.
Notably, the sparse construction~\eqref{eq:para_construct} is one valid choice; other parameter configurations may also realize the same one-step update on structured prompts.


\section{Inference-Time Convergence with Known Dynamics}
\label{sec:population}

In this warm-up section, we analyze the CoT-generated TD iterates under known dynamics, where the transition matrix $P_\pi$ is available at inference time and may be used to construct the context.
Following the setting in prior analyses of inference-time convergence for ICRL \citep{wang2025emergence}, we replace the sampled successor features by their conditional expectations.
Let $m \doteq |\fS|$ and let $(s_0,\ldots,s_{m-1})$ be an enumeration of the states in $\fS$.
For each $i\in\qty[m]$, let $x_{i}\doteq x(s_i)$, we define the expected successor feature
\begin{equation}
\textstyle
    \bar x_i \doteq \E\qty[x(S')\mid S=s_i]=\sum_{j=0}^{m-1}x_jP_\pi[s_i,s_j].
\end{equation}
Using shorthands $R_{i+1}\doteq r(s_i)$, we then construct the initial prompt $\bar Z_0$ as
\begin{equation}\label{eq:bar_z0}
\bar Z_0 \doteq
\begin{bmatrix}
\sqrt{d_\pi(s_0)}x_0 & \cdots & \sqrt{d_\pi(s_{m-1})}x_{m-1} & \0\\
\gamma \sqrt{d_\pi(s_0)}\bar x_0 & \cdots & \gamma \sqrt{d_\pi(s_{m-1})}\bar x_{m-1} & \0 \\
\sqrt{d_\pi(s_0)}R_{1} & \cdots & \sqrt{d_\pi(s_{m-1})}R_{m} & 1\\
\0 & \cdots & \0 & w_0
\end{bmatrix},
\end{equation}
where $d_\pi$ is the stationary distribution of $P_\pi$.
Define $\bar Z_k$ analogously to \eqref{eq:zt_def} by appending the CoT buffer columns while keeping the first $m$ context columns fixed as in \eqref{eq:bar_z0}.
With the same $(P,Q)$ defined in \eqref{eq:para_construct}, applying one CoT step to $\bar Z_k$ and reading out the weight component yields a sequence of deterministic iterates $\{\bar w_k\}$ satisfying (Lemma~\tref{lem:warmupAb})
\begin{equation}\label{eq:expected_update_expanded}
\bar w_{k+1}=\bar w_k+\eta(b-A\bar w_k),
\end{equation}
where $\eta \doteq \alpha/m$ is a constant stepsize inherited from $P$, with the factor $1/m$ arising from the normalization in~\eqref{eq:lsa-layer}.
This recursion has the unique fixed point $w_*$ defined in \eqref{eq:wstar}.
Define the symmetric part $H \doteq \tfrac12(A+A^\top)$ and the spectral constants
\begin{equation}
\label{eq:muL}
\mu \doteq \lambda_{\min}(H),
\quad
L \doteq \|A\|_2^2.
\end{equation}
A standard contraction argument on~\eqref{eq:expected_update_expanded} then yields the geometric convergence in the MSPBE.
\begin{proposition}
\label{prop:expect}
For any constant stepsize $\eta\in(0,\mu/L]$, there exists a constant $C_{\text{Prop}\tref{prop:expect}}$ such that the iterates generated by
\eqref{eq:expected_update_expanded} satisfy for all $k\ge 0$,
\begin{equation}
\label{eq:expect}
    \fL(\bar w_k)\le C_{\text{Prop}\tref{prop:expect}}\bigl(1-\eta\mu)^k \fL(\bar w_0).
\end{equation}
\end{proposition}
The proof is in Section~\tref{proof:expect}.
Proposition~\ref{prop:expect} immediately yields a tradeoff between numbers of steps and precision in the following corollary.
\begin{corollary}
\label{cor:depth}
Fix any constant stepsize $\eta\in(0,\mu/L]$.
For any $\varepsilon\in(0, C_{\text{Prop}\tref{prop:expect}}\fL(\bar w_0)]$, it holds that $\fL(\bar w_k)\le \varepsilon$ whenever
\begin{equation}
\label{eq:depth}
    k \ge \frac{1}{\eta\mu}\log\Big(\frac{C_{\text{Prop}\tref{prop:expect}}\fL(\bar w_0)}{\varepsilon}\Big).
\end{equation}
\end{corollary}
The proof is in Section~\tref{proof:depth}.
Recall $\bar w_0 = \0$, $\fL(\bar w_0)$ is a constant.
Therefore, reaching $\fL(\bar w_k)\le \varepsilon$ requires only $k=O(\log(1/\varepsilon))$ CoT steps for a fixed stepsize $\eta$.

\section{Inference-Time Convergence with Unknown Dynamics}
\label{sec:finite}

Section~\tref{sec:population} analyzes an idealized known-dynamics setting, where the context is constructed using the transition matrix $P_\pi$.
This assumption is restrictive in typical RL settings, where the dynamics are unknown and the model only observes a single finite trajectory $\tau_n=(S_0,R_1,\ldots,S_n)$.
This is also a limitation of prior work on inference-time convergence for ICRL in \citet{wang2025emergence}.
In this section, we analyze the convergence of CoT-driven ICRL as specified in Theorem~\tref{thm:onestep}.
This extends \citet{wang2025emergence} in three aspects:
(i) we move from tabular representations to linear function approximation,
(ii) we replace the state-enumeration construction by an arbitrary trajectory $\tau_n$,
and (iii) we use purely sample-based quantities without expected successor features.
Compared with the known-dynamics setting in \eqref{eq:expected_update_expanded}, the CoT iterates generated under  \eqref{eq:zk-rec} is harder to analyze because they rely on a single Markovian trajectory $\tau_n$, which leads to estimation error under dependence.
Our goal is to quantify how the loss $\fL(w_k)$ decreases with the number of CoT steps $k$ until it reaches a statistical floor controlled by the context length $n$.

Throughout this section, we use the same prompt construction $Z_k$ introduced in \eqref{eq:zt_def}.
We define the empirical moments as
\begin{equation}
\textstyle
\widehat A=\frac1n\sum_{t=0}^{n-1} x_t\bigl(x_t-\gamma x_{t+1}\bigr)^\top,
\quad
\widehat b=\frac1n\sum_{t=0}^{n-1} R_{t+1} x_t .
\label{eq:empirical-moments}
\end{equation}
Equivalently, \eqref{eq:batch-td0} can be written as
\begin{equation}
w_{k+1}=w_k+\alpha\bigl(\widehat b-\widehat A w_k\bigr),
\label{eq:empirical-recursion}
\end{equation}
with $w_0=\0$ as in our prompt construction.
Our finite sample analysis depends on the deviation of $(\widehat A,\widehat b)$ from $(A,b)$ in \eqref{eq:abc}.
Define
\begin{equation}
\label{eq:varepsilon_ab}
\textstyle
    \varepsilon_A=\norm{\widehat A-A}_2,
\,
\varepsilon_b=\norm{\widehat b-b}_2 .
\end{equation}
Since the chain induced by $P_\pi$ is ergodic, we couple $\{S_t\}$ to a stationary chain via Lemma~\ref{lem:doeblin-coupling}, apply a blocking argument with Berbee's coupling (Lemmas~\ref{lem:beta}, \ref{lem:berbee-blocks}) to handle dependence, and conclude with the matrix Bernstein inequality (Lemma~\ref{lem:berstein}). This yields concentration bounds on $(\widehat A, \widehat b)$ in terms of an effective sample size $n_{\text{eff}}$ instead of $n$.
\begin{lemma}
\label{lem:concentration}
For any $\delta\in(0,1)$, there exist constants
$C_{\tref{lem:concentration},1}, C_{\tref{lem:concentration},2}, C_{\tref{lem:concentration},3}>0$
such that letting
\[
\textstyle
m \doteq \Bigl\lceil C_{\tref{lem:concentration},1}
\log\Bigl(\frac{C_{\tref{lem:concentration},2}n}{\delta}\Bigr)\Bigr\rceil,
\quad
n_{\mathrm{eff}} \doteq \Bigl\lfloor \frac{n}{2m}\Bigr\rfloor,
\]
if $n_{\mathrm{eff}}\ge d+\log\frac{4}{\delta}$, then with probability at least $1-\delta$,
$    \varepsilon_A \le \varepsilon_A(\delta),\,
\varepsilon_b \le \varepsilon_b(\delta),$
where
\begin{align}
\textstyle
    \varepsilon_A(\delta)\le C_{\tref{lem:concentration},3}\sqrt{\frac{d+\log(4/\delta)}{n_{\mathrm{eff}}}},\quad
\varepsilon_b(\delta)\le C_{\tref{lem:concentration},3}\sqrt{\frac{d+\log(4/\delta)}{n_{\mathrm{eff}}}}.
\end{align}
\end{lemma}

The proof is in Section~\tref{proof:concentration}.
Define the aggregated deviation
\begin{equation}
\label{eq:defdelta}
\textstyle
    \varepsilon(\delta)\doteq \varepsilon_A(\delta)+\varepsilon_b(\delta).
\end{equation}
Let $\mathcal{E}_\delta$ denote the event in Lemma~\ref{lem:concentration}, which holds with probability at least $1-\delta$.
On $\mathcal{E}_\delta$, \eqref{eq:empirical-recursion} satisfies a contraction inequality with an additive perturbation term controlled by $\varepsilon(\delta)$.
Iterating this inequality for $k$ steps yields the following finite-sample bound.
\begin{theorem}
\label{thm:finite}
Fix any $\delta\in(0,1)$ and $\alpha\in(0,\mu/L]$.
Then there exist constants
$C_{\text{Thm}\tref{thm:finite},1}, C_{\text{Thm}\tref{thm:finite},2}>0$
such that on $\mathcal E_\delta$, if
$n_{\mathrm{eff}}
\ge
C_{\text{Thm}\tref{thm:finite},1}\Bigl(d+\log\frac{4}{\delta}\Bigr)$,
then for all $k\ge 0$,
\begin{equation}
\label{eq:finite-main}
\textstyle
\fL(w_k)
\le
C_{\text{Thm}\tref{thm:finite},2}\qty(\qty(1-\frac{\alpha\mu}{2})^k \fL(w_0)
+\varepsilon(\delta)^2).
\end{equation}
\end{theorem}
The proof is in Section~\tref{proof:finite}.
On $\mathcal E_\delta$, \eqref{eq:finite-main} shows that increasing the CoT length drives the population MSPBE $\fL(w_k)$ down
geometrically until it reaches a statistical floor governed by $\varepsilon(\delta)$.
We next convert \eqref{eq:finite-main} into an explicit CoT length requirement for reaching this floor.
\begin{corollary}
\label{cor:floor-depth}
Under the conditions of Theorem~\ref{thm:finite}, fix any $\delta\in(0,1)$.
If
$k \ \ge\  \frac{2}{\alpha\mu}
\log\qty(
\frac{\fL(w_0)}
     {\varepsilon(\delta)^2}
),$
then on $\mathcal{E}_\delta$,
\begin{equation}
\label{eq:floor-level}
\fL(w_k)\le\ 2C_{\text{Thm}\tref{thm:finite},2}\,\varepsilon(\delta)^2 .
\end{equation}
\end{corollary}
The proof is in Section~\tref{proof:floor-depth}.
Moreover, Lemma~\ref{lem:concentration} implies
$\varepsilon(\delta)=\tilde O\qty(\sqrt{\frac{d+\log(1/\delta)}{n_{\mathrm{eff}}}})$ with $n_{\text{eff}} = \tilde{\Theta}(n)$,
hence the statistical floor in \eqref{eq:floor-level} scales as
$\tilde O\qty(\frac{d+\log(1/\delta)}{n})$, thus vanishes as $n\to\infty$.

\section{Emergence under Reinforcement Pretraining}
\label{sec:emergence}
Previously, we established that the $(P,Q)$ constructed in \eqref{eq:para_construct} is desirable for in-context policy evaluation, as it induces an iterative batch TD update of the form \eqref{eq:batch-td0}.
We now take an optimization perspective and show that these parameters attain global optimality for a pretraining loss over a finite dataset of trajectories.

Recall that the linear self-attention layer in \eqref{eq:lsa-layer} is parameterized by matrices $(P,Q)\in\R[(3d+1)\times (3d+1)]$.
Following a common practice in the in-context learning theory literature \citep{vonoswald2023transformers, ahn2024transformers,mahankali2023one,zhang2023trained,cheng2024fgd,gatmiry2024looped,huang2025transformers,wang2025emergence}, we restrict the learnable parameters $\theta$ to the block-sparse pattern defined as
\begin{align}
\textstyle
\label{eq:Theta-CoT}
\Theta^{\mathrm{CoT}} \doteq \Bigg\{\theta = (P, Q)\bigg| 
    P =
    \begin{bmatrix}
        0_{(2d+1) \times d} & 0_{(2d+1) \times (2d+1)} \\
        P_1 & 0_{d \times (2d+1)}
    \end{bmatrix}, Q = 
    \begin{bmatrix}
        0_{2d \times 2d} & 0_{2d \times 1} & \tilde Q \\
        0_{1 \times 2d} & 1 & 0_{1 \times d} \\
        0_{d \times 2d} & 0_{d \times 1} & 0_{d \times d}
    \end{bmatrix}
\Bigg\},
\end{align}
where $\tilde Q \doteq \begin{bmatrix} Q_1\\Q_2 \end{bmatrix}$ with $Q_1, Q_2 \in \R[d\times d]$.
Under CoT autoregressive generation, all steps share the same parameter $\theta$.
In our parameterization, $\theta$ has $3d^2$ degrees of freedom given by the free blocks $(P_1,Q_1,Q_2)$.
Throughout this section, we interpret $\nabla_\theta(\cdot)$ as the gradient with respect to the vectorized entries of $(P_1,Q_1,Q_2)$.
Let $\theta_*\in\Theta^{\mathrm{CoT}}$ denote the constructed choice in \eqref{eq:para_construct}, which yields the batch TD update in Theorem~\ref{thm:onestep}.

We consider a finite pretraining dataset $\mathcal D \doteq \qty{\tau_n^{(i)}}_{i=0}^{N_{\mathrm{train}}-1}$ with $N_{\mathrm{train}}$ trajectories sampled from the same MRP under a fixed policy $\pi$, where each trajectory is $\tau_n^{(i)}=(S_0^{(i)},R_1^{(i)},\dots,S_n^{(i)})$.
Fix any $\theta\in\Theta^{\mathrm{CoT}}$ and an index $i\in\qty[N_{\mathrm{train}}]$.
Let $\tau_n \doteq \tau_n^{(i)}$.
For notational simplicity, we suppress the superscript $(i)$ and write $\tau_n=(S_0,R_1,\dots,S_n)$.
Let $w_k(\theta;\tau_n)$ denote the weight iterate produced after $k$ CoT steps.
For any $j\in\{0,1,\dots,n-1\}$, the TD error at step $k$ is
\begin{align}
\delta_{k,j}(\theta;\tau_n)
\doteq R_{j+1}+\gamma x_{j+1}^\top w_k(\theta;\tau_n)-x_j^\top w_k(\theta;\tau_n).
\end{align}
With the value estimate $\hat v_k(s;\theta,\tau_n)\doteq x(s)^\top w_k(\theta;\tau_n)$, we define the empirical TD update direction at step $k$ as
\begin{equation}
\label{eq:def-Delta-k}
\textstyle
\Delta_k(\theta;\tau_n)
\doteq
\frac1n\sum_{j=0}^{n-1}
\delta_{k,j}(\theta;\tau_n)
\nabla_\theta \hat v_k(S_j;\theta,\tau_n)
\in\R[3d^2].
\end{equation}
A classical TD objective is the norm of the expected update (NEU)~\citep{sutton2009convergent, sutton2009fast}, originally defined as the squared norm of $\mathbb{E}_{d_\pi}[\delta(w) \nabla v_w(S)]$ under the stationary distribution $d_\pi$, with zeros characterizing TD fixed points.
Since we operate in the unknown-dynamics regime with a single Markovian trajectory $\tau_n$, the stationary expectation is not directly accessible, and we replace it by a trajectory average over $\tau_n$, giving the empirical update norm (EUN) loss
\begin{equation}
\label{eq:EUN}
J_k(\theta;\tau_n)\doteq \|\Delta_k(\theta;\tau_n)\|_2.
\end{equation}
A zero of $J_k(\cdot;\tau_n)$ corresponds to $\Delta_k(\theta;\tau_n) = 0$, i.e., a batch TD fixed point on $\tau_n$, which makes the EUN a natural pretraining objective for studying whether the desirable parameter $\theta_*$ in~\eqref{eq:para_construct} attains global optimality.

Aggregating over $\mathcal D$, we define the empirical pretraining objective at step $k$ as
\begin{equation}
\textstyle
\label{eq:dataset-triangle}
J_k(\theta;\mathcal D)
\doteq
\frac1{N_{\mathrm{train}}}\sum_{\tau_n\in\mathcal D} J_k(\theta;\tau_n),
\end{equation}
where $\sum_{\tau_n\in\mathcal D}$ abbreviates $\sum_{i=0}^{N_{\mathrm{train}}-1}$ 
when the trajectory is a dummy variable. We aim to show that $\theta_*$ globally 
minimizes $J_k(\cdot;\mathcal D)$ as $k\to\infty$.

\begin{theorem}
\label{thm:emergence}
Fix $\delta\in(0,1)$ and a stepsize $\alpha\in\bigl(0,\mu/(8L)\bigr]$.
There exists a constant $C_{\text{Thm}\tref{thm:emergence}}>0$ such that if for each trajectory $\tau_n^{(i)} \in \mathcal D$,
\begin{equation}
\label{eq:neff}
\textstyle
    n_{\mathrm{eff}} \geq 
C_{\text{Thm}\tref{thm:emergence}}
\left(d + \log \tfrac{4N_{\mathrm{train}}}{\delta}\right),
\end{equation}
then with probability at least $1-\delta$, $\lim_{k\to\infty} J_k(\theta_*;\mathcal D)=0$.
\end{theorem}
The proof is in Section~\tref{proof:emergence}.
Since $J_k(\cdot;\mathcal D)\ge 0$ by definition, Theorem~\tref{thm:emergence} implies that $\theta_*$ attains the global minimum of the pretraining objective in the limit $k\to\infty$.
To interpret this limit, for each trajectory $\tau_n^{(i)}\in\mathcal D$ define $\widehat w_*^{(i)}$ by $\widehat A^{(i)}\widehat w_*^{(i)}=\widehat b^{(i)}$.
Under the same conditions, we obtain the following corollary.
\begin{corollary}
\label{cor:bridge}
Fix $\delta\in(0,1)$ and a stepsize $\alpha\in\bigl(0,\mu/(8L)\bigr]$.
If \eqref{eq:neff} holds, then with probability at least $1-\delta$, for all $i\in[N_{\mathrm{train}}]$,
\begin{equation}
\label{eq:bridge-floor}
\fL(\widehat w_*^{(i)})
\le
C_{\text{Thm}\tref{thm:finite},2}\,\varepsilon(\delta')^2,
\end{equation}
where $\delta' = \delta/N_{\mathrm{train}}$.
\end{corollary}
The proof is in Section~\tref{proof:bridge}.
Consequently, under the $\theta_*$ parametrization, the pretraining objective recovers the per-trajectory batch-TD fixed points, and the remaining population MSPBE error relative to $w_*$ is purely statistical, governed by $\varepsilon(\delta')$ and vanishing as $n\to\infty$.

\section{Related Works}
\label{sec:related}
In-context reinforcement learning (ICRL) studies agents that adapt at inference time without parameter updates by conditioning on additional interaction history.
ICRL is often viewed as a form of black-box meta-RL \citep{duan2016rl,wang2016learning,beck2025tutorial}, and it can also be regarded as a special case of in-context learning (ICL) in the broad sense that learning happens in the forward pass of a fixed network \citep{brown2020incontext,zhang2023trained,dong2024survey}.
A useful categorization of ICRL pretraining is supervised pretraining versus reinforcement pretraining \citep{moeini2025survey}.

Supervised pretraining is typically imitation or algorithm distillation \citep{laskin2023context}, and has been instantiated through many data constructions that expose within-context learning progress \citep{laskin2023context,liu2023emergent,shi2023cross,kirsch2023towards,zisman2024emergence,huang2024rlcot,huang2024decision,dai2024incontext,zisman2025ngram,polubarov2025vintix}.
In this setting, in-context algorithmic behavior is expected since training explicitly imitates an RL learning procedure.
Theory shows that such objectives can induce in-context implementations of decision-making procedures under suitable assumptions \citep{lee2024supervised,lin2023transformers}.
We instead study reinforcement pretraining and focus on the inference-time dynamics induced by autoregressive chain-of-thought (CoT) token generation for policy evaluation.

In terms of reinforcement pretraining, a long line of work shows that training agents with RL objectives can induce inference-time learning behaviors from interaction history \citep{duan2016rl,wang2016learning,mishra2018simple,ritter2018been,stadie2019some,zintgraf2020varibad,melo2022transformer}.
Related empirical work investigates scalable architectures and training recipes for ICRL and its variants \citep{team2023human,lu2023structured,grigsby2024amago,grigsby2024amago2,elawady2024relic,xu2024meta,cook2024artificial,liu2025locoformer,moeini2025safe}.
On the theory side, \citet{park2024llm} studies ICRL with large language models (LLMs) and proposes a regret-loss for reinforcement pretraining.
The closest works related to us are \citet{wang2025ictd, wang2025emergence}, who prove that the layerwise forward pass of linear Transformers implements step-by-step TD updates and establish inference-time convergence and emergence of those parameters under a suitable prompt.
Our work differs from \citet{wang2025ictd,wang2025emergence} in that
(i) we analyze autoregressive Chain-of-Thought (CoT) generation in a one-layer Transformer, rather than layerwise inference dynamics across depth;
(ii) beyond the idealized known-dynamics regime studied in \citet{wang2025emergence}, we establish convergence under unknown dynamics with a single Markovian trajectory reused across generations;
(iii) we further show that our constructed parameterization is learnable under certain reinforcement pretraining objective even under unknown dynamics, which is not covered by \citet{wang2025emergence}.


\looseness=-1
A concurrent submission~\citep{xie2026softmax} also extends \citet{wang2025ictd, wang2025emergence}, but takes a fundamentally different route: it analyzes softmax attention with depth-based computation, where each layer implements a novel weighted softmax TD algorithm in a tabular value representation, while we analyze linear attention with autoregressive CoT and linear function approximation. As a result, the two works draw on largely disjoint technical tools. For convergence, our analysis is finite-sample on a single Markovian trajectory, whereas \citet{xie2026softmax} primarily analyze convergence at the population level under a structural assumption on the kernel. For emergence, our analysis controls how the algorithmic iterates depend on the parameters, whereas \citet{xie2026softmax} control how the value prediction depends on the parameters. Both works follow \citet{wang2025ictd} for the Boyan's chain \citep{boyan1999least} experimental setup.

\looseness=-1
Beyond ICRL, our finite-sample analysis in Section~\ref{sec:finite} is technically related to the literature on TD with linear function approximation, most of which analyze online TD or its averaged variants under Markovian sampling~\citep{bhandari2018finite, srikant2019finite, patil2023finite, samsonov2024improved, lee2025finite}; closer to our setting, \citet{lazaric2010finite} analyze pathwise LSTD and \citet{prashanth2021concentration} analyze a stochastic approximation variant of LSTD with uniform resampling from a fixed batch. Our work differs from this line in that batch TD reuses the trajectory across CoT steps with full-batch averaging, so the iterates are determined by the empirical moments rather than online noise.

\looseness=-1
CoT prompting was introduced as a practical technique to elicit multi-step reasoning in language models \citep{wei2022chain}.
Beyond standard CoT prompting, inference-time computation can be increased by exploring multiple reasoning trajectories, such as self-consistency and tree-of-thought search \citep{wang2023self,yao2023tot}.
Recent theory formalizes multi-step in-context algorithms by viewing each generated step or loop as one iteration \citep{vonoswald2023transformers,ahn2024transformers,gatmiry2024looped,huang2025transformers}, mostly in the ICL setting.
In ICRL, hierarchical decision-making with a CoT has been used to compress long interaction histories \citep{huang2024rlcot}, while theoretical characterizations of autoregressive CoT generation remain limited.
To our knowledge, we provide the first theoretical characterization of how autoregressive CoT generation affects in-context policy evaluation in ICRL.

\section{Experiments}
\label{sec:experiment}
\textbf{Setup.}
Following \citet{wang2025ictd}, we conduct experiments on Boyan's chain \citep{boyan1999least},
a canonical environment for evaluating RL algorithms.
Prompts are constructed following the definition of $Z_0$ in \eqref{eq:z0},
and the model is trained on $5000$ sampled MRPs, where each parameter update unrolls the model autoregressively with shared parameters $(P, Q)$ across $16$ CoT iterations.
Notably, all elements of $P$ and $Q$ are equally trainable.
Our full training procedure is stated as Algorithm~\ref{alg:cot-td}.
Unless otherwise stated, we follow \citet{wang2025ictd} for task construction and hyperparameters,
with additional training and implementation details deferred to Appendix~\ref{sec:aux_experiment}.
\begin{figure}[t]
  \centering

  \begin{minipage}[c]{0.495\textwidth}
    \centering
    \includegraphics[width=0.92\linewidth]{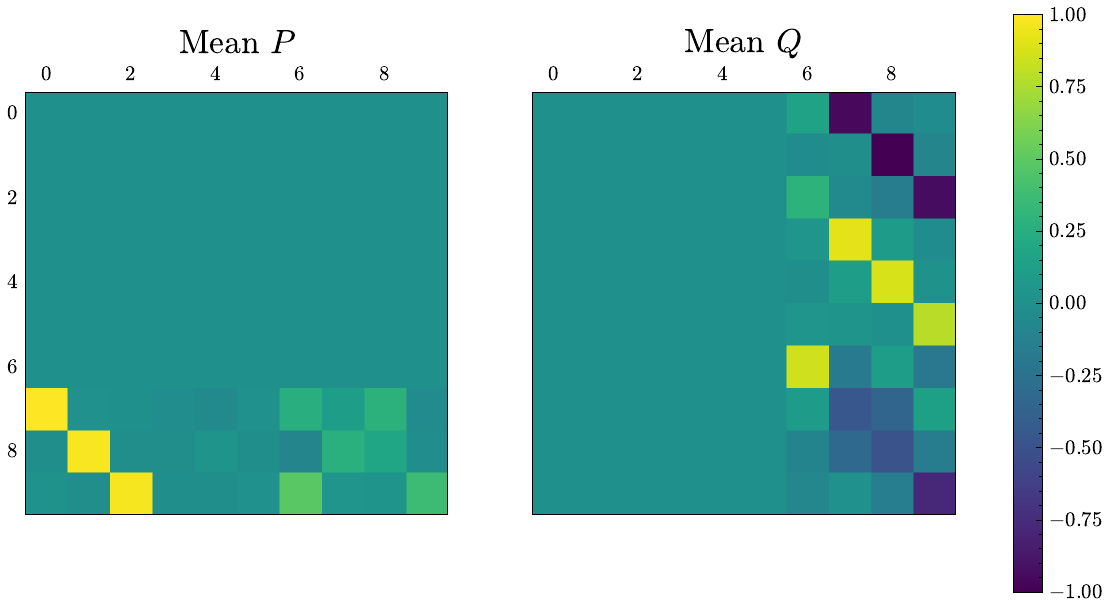}\\[2pt]
    \small (a) Learned parameters $P$ and $Q$.
  \end{minipage}
  \hfill
  \begin{minipage}[c]{0.495\textwidth}
    \centering
    \begin{minipage}[t]{0.49\linewidth}
      \centering
      \includegraphics[width=\linewidth]{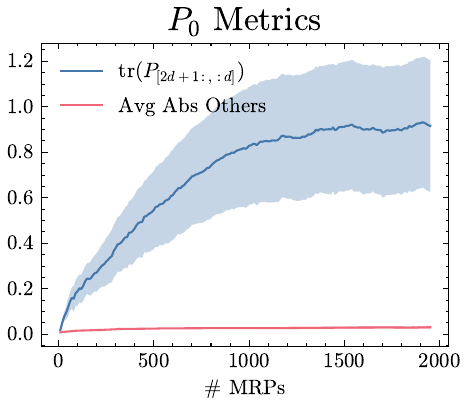}
    \end{minipage}
    \hfill
    \begin{minipage}[t]{0.49\linewidth}
      \centering
      \includegraphics[width=\linewidth]{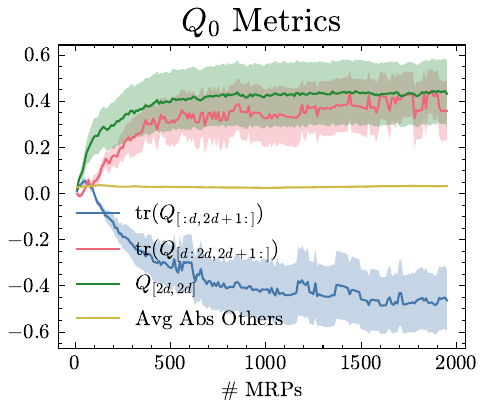}
    \end{minipage}\\[2pt]
    \small (b) Element-wise learning progress of $P$ and $Q$.
  \end{minipage}

  \caption{Learned parameters and element-wise learning progress. In (a), $P$ and $Q$ are shown as elementwise means over $10$ seeds, and each matrix is normalized by its maximum absolute entry. In (b), shaded regions denote standard errors across $10$ seeds.}
  \label{fig:pq-combined}
\end{figure}



\textbf{TD Emergence.}
Figure~\ref{fig:pq-combined}(a) reports the learned parameters $(P,Q)$
averaged over seeds, which exhibit a clear block pattern close to \eqref{eq:para_construct}.
Figure~\ref{fig:pq-combined}(b) clearly shows the element-wise learning progress of $P$ and $Q$, where only the blocks prescribed by \eqref{eq:para_construct} strengthen during training while all other entries remain small.

\section{Conclusion}
\label{sec:conclusion}
This paper provides a white-box account of how CoT generation implements inference-time policy evaluation in linear Transformers.
Under a simple prompt construction, each CoT step realizes one batch TD update on the context.
We establish geometric convergence guarantees in two regimes, namely MSPBE contraction under known dynamics and geometric improvement under unknown dynamics that saturates at a statistical floor governed by an effective sample size.
We further connect the CoT computation to reinforcement pretraining by showing that the constructed parameterization is one of the global minimizers of the empirical pretraining objective.
Controlled experiments on Boyan's chain demonstrate the emergence of the predicted block structure across multiple policy-evaluation tasks.
Limitations include linear self-attention and the small scale of our empirical study, and addressing them may further advance the theoretical understanding of in-context reinforcement learning.

\section*{Acknowledgments}
This work is supported in part by the US National Science Foundation under the awards III-2128019, SLES-2331904, and CAREER-2442098, the Commonwealth Cyber Initiative's Central Virginia Node under the award VV-1Q26-001, and a Cisco Faculty Research Award.

{\small
\bibliography{bibliography}
}

\newpage
\appendix
\onecolumn
\section{Auxiliary Lemmas and Definitions}
\label{sec:aux_lemma}
\begin{lemma}[Lemma 6 of \citet{tsitsiklis1997analysis}]\label{lem:A4:proj_bellman}
Let $\Pi \doteq X(X^\top D_\pi X)^{-1}X^\top D_\pi$ be the $D_\pi$-orthogonal projection onto $\mathrm{span}(X)$,
and let $\mathcal{T}^\pi$ be the Bellman expectation operator.
Then $\Pi\mathcal{T}^\pi$ is a $\gamma$-contraction under $\|\cdot\|_{D_\pi}$:
\[
\|\Pi\mathcal{T}^\pi v-\Pi\mathcal{T}^\pi v'\|_{D_\pi}
\le
\gamma\|v-v'\|_{D_\pi},
\qquad \forall v,v'\in\mathbb{R}^{|\mathcal{S}|}.
\]
Consequently, there exists a unique fixed point $v_*\in\mathrm{span}(X)$ such that
\[
v_*=\Pi\mathcal{T}^\pi v_*.
\]
Moreover, if $X$ has full column rank, there exists a unique $w_*\in\mathbb{R}^d$ such that $v_*=Xw_*$.
\end{lemma}

\begin{lemma}[Theorem~3.7 of \citet{bradley2005mixing}]
\label{lem:beta}
Consider the stationary Markov chain with transition matrix $P_\pi$ and
stationary distribution $d_\pi$.
Define the $\beta$-mixing coefficients for $m\ge 1$ by
\[
\beta(m)
\doteq
\sup_{t\ge 0}
\mathbb E\Bigg[
\sup_{A\in\sigma(S_{t+m},S_{t+m+1},\ldots)}
\bigl|\mathbb P(A\mid\sigma(S_0,\ldots,S_t))-\mathbb P(A)\bigr|
\Bigg].
\]
Then the following statements are equivalent.
\begin{enumerate}
\item \textbf{geometric ergodicity.}
There exist constants $R<\infty$ and $\rho\in(0,1)$ such that for all $s\in\fS$ and all $n\ge 0$,
\begin{equation}
\label{eq:ergodic}
    \bigl\|P_\pi^n(s,\cdot)-d_\pi(\cdot)\bigr\|_{\mathrm{TV}}
\le
R\rho^n.
\end{equation}
\item 
\textbf{$\beta$-mixing.}
There exist constants $c_\beta\ge 1$ and $\rho\in(0,1)$ such that for all $m\ge 1$,
\[
\beta(m)\le c_\beta\rho^m .
\]
\end{enumerate}
\end{lemma}


\begin{lemma}[Proposition~1.7 of \citet{levin2017markov}]
\label{lem:aperiodic-positive}
If $P_\pi$ is aperiodic and irreducible, then there exists an 
integer $r_0$ such that $P_\pi^r(s,s') > 0$ for all 
$s, s' \in \fS$ and $r \ge r_0$.
\end{lemma}


\begin{lemma}[Lemma 7 in \citet{steffen2024mixing}]\label{lem:berbee-blocks}
\label{lem:berbee}
Let $X_i$, $i\in\mathbb{N}$ be a stationary sequence of random variables with $\beta$-mixing coefficients $\beta(j)$, $j\in\mathbb{N}$.
For fixed $k,i\in\mathbb{N}$, define the block vectors
\[
Y_i \doteq \bigl(X_{i(k+i)}, X_{i(k+i)+1}, \ldots, X_{i(k+i)+i-1}\bigr)\in\mathbb{R}^{i},\qquad i\in\mathbb{N}.
\]
There exists a sequence of independent random variables $Y_i^{\ast}$, $i\in\mathbb{N}$, taking values in $\mathbb{R}^{i}$ and having the same distribution as $Y_i$ such that, for every $i\in\mathbb{N}$,
\[
\Pr\bigl(Y_i^{\ast}\neq Y_i\bigr)\le\beta(k).
\]
\end{lemma}

\begin{definition}[Section 2.10 in \citet{tropp2015expected}]
\label{def:hermitian}
Let $B \in \mathbb{R}^{d_1 \times d_2}$ and define
\[
  \mathcal H(B) \doteq 
  \begin{bmatrix}
    0 & B \\
    B^\top & 0
  \end{bmatrix}.
\]
Then $\mathcal H(B)$ is symmetric,
\begin{equation}
\label{eq:herm}
      \mathcal H(B)^2=
  \begin{bmatrix}
    BB^\top & 0\\
    0 & B^\top B
  \end{bmatrix},
  \qquad
  \|\mathcal H(B)\|_2=\|B\|_2 .
\end{equation}
\end{definition}

\begin{lemma}[Theorem 1.6 in \citet{tropp2012user}]
\label{lem:berstein}
Let $\{Z_k\}_{k=1}^N$ be a finite sequence of independent random matrices in
$\mathbb{R}^{d_1\times d_2}$.
Assume that each random matrix satisfies
\[
  \mathbb{E}[Z_k]=0,
  \qquad
  \|Z_k\|_2 \le R
  \quad \text{a.s.}
\]
Define the variance parameter
\[
  \sigma^2
  \doteq
  \max\left\{
    \left\|\sum_{k=1}^N \mathbb{E}\!\left[ Z_k Z_k^\top \right]\right\|_2,\,
    \left\|\sum_{k=1}^N \mathbb{E}\!\left[ Z_k^\top Z_k \right]\right\|_2
  \right\}.
\]
Then for all $t\ge 0$,
\[
  \Pr\!\left(
    \left\|\sum_{k=1}^N Z_k\right\|_2 \ge t
  \right)
  \le
  (d_1+d_2)\,
  \exp\!\left(
    -\frac{t^2/2}{\sigma^2 + Rt/3}
  \right).
\]
\end{lemma}

\section{Proofs in Section~\tref{sec:algo}}
\label{sec:proofsec3}
\subsection{Proof of Theorem~\tref{thm:onestep}}
\label{proof:onestep}
\begin{proof}
Recall $Z_k$ in \eqref{eq:zt_def}. Fix any $k\ge 0$ and let
$z_i \doteq (Z_k)_{[:,i]}$ denote the $i$-th token in $Z_k$.
By \eqref{eq:lsa-layer}, we have
\begin{equation}
\label{eq:last_token_update_sum}
\fF(Z_k;P,Q)_{:,-1}
=
z_{n+k} + \frac{1}{n}\sum_{i=0}^{n+k} (Pz_i)\,(z_i^\top Q z_{n+k}).
\end{equation}
By the definition of $z_{n+k}$ and the block structure of $Q$ in \eqref{eq:para_construct},
for each context index $j\in\{0,\ldots,n-1\}$ we have
\begin{align}
\label{eq:zqz}
z_j^\top Q z_{n+k}
&=
\begin{bmatrix}
x_j^\top & (\gamma x_{j+1})^\top & R_{j+1} & \0^\top
\end{bmatrix}
\begin{bmatrix}
    0_{d \times 2d} & 0_{d \times 1} & -I_d \\
    0_{d \times 2d} & 0_{d \times 1} & I_d \\
    0_{1 \times 2d} & 1            & 0_{1 \times d} \\
    0_{d \times 2d} & 0_{d \times 1} & 0_{d \times d}
\end{bmatrix}
\begin{bmatrix}
\0\\
\0\\
1\\
w_k
\end{bmatrix} \notag\\
&= R_{j+1} + \gamma\langle w_k, x_{j+1}\rangle - \langle w_k, x_j\rangle.
\end{align}
Moreover, by the definition of $P$ in \eqref{eq:para_construct}, for each context index
$j\in\{0,\ldots,n-1\}$,
\begin{align}
\label{eq:pzi}
Pz_j
&=
\begin{bmatrix}
0_{(2d+1) \times d} & 0_{(2d+1) \times (2d+1)} \\
\alpha I_d          & 0_{d \times (2d+1)}
\end{bmatrix}
\begin{bmatrix}
x_j\\
\gamma x_{j+1}\\
R_{j+1}\\
\0
\end{bmatrix}
=
\begin{bmatrix}
\0\\
\0\\
0\\
\alpha x_j
\end{bmatrix}.
\end{align}
Moreover, $Pz_{n+i}=\0$ for all $i\in\{0,\ldots,k\}$.
Substituting \eqref{eq:pzi} and \eqref{eq:zqz} into \eqref{eq:last_token_update_sum} yields
\[
\fF(Z_k;P,Q)_{:,-1}
=
\begin{bmatrix}\0\\\0\\1\\ w_k + \frac{\alpha}{n}\sum_{j=0}^{n-1}\delta_j(w_k)\,x_j \end{bmatrix}.
\]
Taking the $w$-block gives
\[
w_{k+1}
= w_k + \frac{\alpha}{n}\sum_{j=0}^{n-1}\qty(R_{j+1}+\gamma x_{j+1}^\top w_k - x_j^\top w_k)x_j,
\]
which completes the proof.
\end{proof}

\section{Proofs in Section~\tref{sec:population}}
\label{sec:proofsec4}
\subsection{Proof of Lemma~\tref{lem:warmupAb}}
\begin{lemma}
\label{lem:warmupAb}
With $(P,Q)$ defined in \eqref{eq:para_construct}, one CoT step applied to $\bar Z_k$ yields a deterministic weight iterate $\bar w_{k+1}$ that satisfies
\[
\bar w_{k+1} = \bar w_k + \eta(b - A\bar w_k),
\]
where $(A,b)$ is defined in \eqref{eq:abc} and $\eta=\alpha/m$.
\end{lemma}

\begin{proof}
By Theorem~\tref{thm:onestep} and \eqref{eq:bar_z0}, one CoT step on $\bar Z_k$ updates the weight component in the form
\[
\bar w_{k+1} = \bar w_k + \alpha\bigl(\bar b - \bar A \bar w_k\bigr),
\]
where 
\[
\bar A = \frac1m\sum_{i=0}^{m-1} d_\pi(s_i)x(s_i)\bigl(x(s_i)-\gamma \bar x_i\bigr)^\top = \tfrac1m A,
\qquad
\bar b = \frac1m\sum_{i=0}^{m-1} d_\pi(s_i)r(s_i)x(s_i)=\tfrac1m b.
\]
Both equations are obtained since $(s_0,\ldots,s_{m-1})$ enumerates $\fS$.
Plugging back gives
\[
\bar w_{k+1}
= \bar w_k + \alpha\Bigl(\frac1m b - \frac1m A\bar w_k\Bigr)
= \bar w_k + \eta(b-A\bar w_k),
\]
where $\eta=\alpha/m$.
\end{proof}

\begin{lemma}
\label{lem:mu-positive}
For $A$ in \eqref{eq:abc}, let $H = \frac{1}{2}(A+A^\top)$. 
There exists a constant $C_\tref{lem:mu-positive}>0$ such that 
\begin{equation}
    \mu = \lambda_{\min}(H) \ge (1-\gamma)C_\tref{lem:mu-positive}.
\end{equation}
\end{lemma}
\begin{proof}
Fix any $v\in\mathbb{R}^d$ and let $u=Xv\in\mathbb{R}^{|\mathcal S|}$.
Then
\[
v^\top H v
= \frac12 v^\top(A+A^\top)v
= \frac12 u^\top\Big(D_\pi(I-\gamma P_\pi)+(I-\gamma P_\pi)^\top D_\pi\Big)u.
\]
Expanding the symmetric part gives
\[
D_\pi(I-\gamma P_\pi)+(I-\gamma P_\pi)^\top D_\pi
= 2D_\pi-\gamma(D_\pi P_\pi + P_\pi^\top D_\pi).
\]
By Cauchy--Schwarz and stationarity,
\begin{equation}
    u^\top D_\pi P_\pi u
=\mathbb{E}[u(S_t)u(S_{t+1})]
\le \sqrt{\mathbb{E}[u(S_t)^2]\mathbb{E}[u(S_{t+1})^2]}
= \mathbb{E}[u(S_t)^2]
= u^\top D_\pi u
\end{equation}
and similarly $u^\top P_\pi^\top D_\pi u \le u^\top D_\pi u$.
Therefore,
\[
u^\top\big(2D_\pi-\gamma(D_\pi P_\pi + P_\pi^\top D_\pi)\big)u
\ge 2(1-\gamma)\, u^\top D_\pi u.
\]
Combining the inequalities yields
\begin{equation}
    v^\top H v \ge (1-\gamma)\, u^\top D_\pi u = (1-\gamma)\, v^\top C v \ge (1-\gamma)\lambda_{\min}(C)\|v\|_2^2.
\end{equation}
By \eqref{eq:abc}, $C = X^\top D_\pi X$ is positive definite since $X$ has full column rank and $D_\pi$ has positive diagonal entries by ergodicity, so $\lambda_{\min}(C) > 0$.
Since the above holds for all $v\in\mathbb{R}^d$, choosing $C_\tref{lem:mu-positive}\doteq \lambda_{\min}(C)$ completes the proof.
\end{proof}



\subsection{Proof of Proposition~\tref{prop:expect}}
\label{proof:expect}
\begin{proof}
We first relate $\fL(\bar w_k)$ to the squared error $\|\bar w_k - w^*\|_2^2$. Since
\begin{equation}
\fL(w) = \|C^{-\frac12} A(w - w^*)\|_2^2 \le \|C^{-\frac12} A\|_2^2 \|w - w^*\|_2^2,
\end{equation}
bounding $\|\bar w_k - w^*\|_2^2$ immediately yields an upper bound on $\fL(\bar w_k)$. Subtracting $w^*$ from both sides of~\eqref{eq:expected_update_expanded} and taking the squared norm gives
\begin{align}
\label{eq:eupdate}
\|\bar w_{k+1} - w^*\|_2^2
&= \|\bar w_k - w^*\|_2^2 - 2\eta(\bar w_k - w^*)^\top A(\bar w_k - w^*) + \eta^2 \|A(\bar w_k - w^*)\|_2^2.
\end{align}
For any $v \in \R[d]$, $v^\top H v \ge \mu \|v\|_2^2$ and $\|Av\|_2^2 \le L\|v\|_2^2$, so~\eqref{eq:eupdate} implies
\begin{equation}
\label{eq:one_step_contract}
\|\bar w_{k+1} - w^*\|_2^2 \le (1 - 2\eta\mu + \eta^2 L)\|\bar w_k - w^*\|_2^2.
\end{equation}
When $\eta \le \mu/L$, we have $1 - 2\eta\mu + \eta^2 L \le 1 - \eta\mu < 1$, yielding a contraction.
By~\eqref{eq:one_step_contract}, we have
\[
\|\bar w_{k+1}-w^*\|_2^2 \le (1-2\eta\mu+\eta^2 L)\|\bar w_k-w^*\|_2^2.
\]
By Lemma~\tref{lem:mu-positive}, we have $\mu > (1-\gamma)C_\tref{lem:mu-positive} > 0$, and hence for any stepsize
$\eta\le \mu/L$ we have $1-2\eta\mu+\eta^2 L \le 1-\eta\mu < 1$.
Iterating over $k$ gives
\begin{equation}
\label{eq:witerate}
    \|\bar w_k - w^*\|_2^2 \le (1-\eta\mu)^k \|\bar w_0 - w^*\|_2^2.
\end{equation}
Since $\fL(\bar w_k) = \|C^{-\frac12}A(\bar w_k - w^*)\|_2^2$, applying \eqref{eq:witerate} we have
\begin{align}
\fL(\bar w_k) 
&\leq \|C^{-\frac12}A\|_2^2 (1-\eta\mu)^k \|\bar w_0 - w^*\|_2^2\\
&\le C_{\text{Prop}\tref{prop:expect}}(1-\eta\mu)^k \fL(\bar w_0),
\end{align}
where $C_{\text{Prop}\tref{prop:expect}} \doteq \frac{\|C^{-\frac12}A\|_2^2} {\sigma_{\min}(C^{-\frac12}A)^2}$.
This completes the proof.
\end{proof}


\subsection{Proof of Corollary~\ref{cor:depth}}
\label{proof:depth}
\begin{proof}
By Proposition~\tref{prop:expect}, for any constant stepsize $\eta\in(0,\mu/L]$ we have
\[
\fL(\bar w_k)\le C_{\text{Prop}\tref{prop:expect}}(1-\eta\mu)^k \fL(\bar w_0), \qquad \forall k\ge 0.
\]
Since $\eta\mu\in(0,1]$, we can use $\log(1-x)\le -x$ for $x\in(0,1]$ to obtain
\[
(1-\eta\mu)^k \le \exp(-\eta\mu\,k).
\]
Therefore,
\[
\fL(\bar w_k)\le C_{\text{Prop}\tref{prop:expect}}\exp(-\eta\mu\,k)\,\fL(\bar w_0).
\]
It follows that $\fL(w_k)\le \varepsilon$ whenever
\[
k \ge \frac{1}{\eta\mu}\log\Big(\frac{C_{\text{Prop}\tref{prop:expect}}\fL(\bar w_0)}{\varepsilon}\Big),
\]
which completes the proof.
\end{proof}

\section{Proofs in Section~\tref{sec:finite}}
\label{sec:proofsec5}

\begin{lemma}[Doeblin coupling for finite ergodic chains]
\label{lem:doeblin-coupling}
Let $P_\pi$ be an irreducible and aperiodic transition matrix on 
a finite state space $\fS$ with stationary distribution $d_\pi$. 
For any initial distribution $p_0$, there exists a coupling 
$(S_t, S_t')_{t\ge 0}$ with $S_0\sim p_0$, $S_0'\sim d_\pi$, 
both marginally Markov with kernel $P_\pi$, such that the 
coalescence time 
$\tau_{\mathrm{coal}}\doteq\inf\{t: S_s = S_s' 
\text{ for all } s\ge t\}$ satisfies
\[
\Pr(\tau_{\mathrm{coal}} > t) \;\le\; C_{\tref{lem:doeblin-coupling},1}C_{\tref{lem:doeblin-coupling},2}^{t},
\]
for some constants $ C_{\tref{lem:doeblin-coupling},1}\ge 1$ and $C_{\tref{lem:doeblin-coupling},2}\in(0,1)$.
\end{lemma}
\begin{proof}
Since $\fS$ is finite and $P_\pi$ is irreducible and aperiodic, 
by Lemma~\tref{lem:aperiodic-positive} there exists $r\ge 1$ such that 
$P_\pi^r(s,s')>0$ for all $s,s'\in\fS$. Define
\[
\varepsilon \doteq \sum_{s'\in\fS}\min_{s\in\fS} P_\pi^r(s,s')
>0,
\qquad
\nu(s') \doteq \varepsilon^{-1}\min_{s\in\fS} P_\pi^r(s,s').
\]
Then for each $s\in\fS$, define
\[
Q_s(\cdot) \doteq \frac{P_\pi^r(s,\cdot) - \varepsilon\,\nu(\cdot)}{1-\varepsilon},
\]
so that $P_\pi^r(s,\cdot) = \varepsilon\,\nu(\cdot) + (1-\varepsilon)\,Q_s(\cdot)$, 
where $Q_s$ is a probability distribution on $\fS$ for each $s$.
Construct the split coupling as in Section 4.2 of \citet{roberts2004general}): 
at block times $0, r, 2r,\ldots$, draw a common 
$\mathrm{Bernoulli}(\varepsilon)$ coin; if it succeeds, set 
both chains' next state to a common draw from $\nu$; if it 
fails, update using the residual kernels $Q_{S_t}$ and 
$Q_{S_t'}$ independently. Once the chains meet, use identical 
transitions thereafter. Each block has conditional coupling 
probability at least $\varepsilon$, so 
$\Pr(\tau_{\mathrm{coal}} > kr) \le (1-\varepsilon)^k$ for 
$k\ge 0$. Setting $C_{\tref{lem:doeblin-coupling},2} = (1-\varepsilon)^{1/r}$ and 
$C_{\tref{lem:doeblin-coupling},1} = (1-\varepsilon)^{-1}$ completes the proof.
\end{proof}

\subsection{Proof of Lemma~\tref{lem:concentration}}
\label{proof:concentration}

\begin{proof}
Fix $\delta\in(0,1)$. We first prove the bound for $\varepsilon_A$ defined in \eqref{eq:varepsilon_ab}.
Define the stationary mean $M_\pi \doteq \E_{d_\pi}[M_0] = A$. Then
\begin{equation}
\label{eq:DeltaA}
\Delta_A \doteq \widehat A - A 
= \underbrace{\frac{1}{n}\sum_{t=0}^{n-1}(M_t - M_t')}_{T_1}
+ \underbrace{\frac{1}{n}\sum_{t=0}^{n-1}(M_t' - M_\pi)}_{T_2}.
\end{equation}

\paragraph{Coupling to a stationary chain.}
By Lemma~\tref{lem:doeblin-coupling}, there exists a stationary 
chain $\{S_t'\}$ with $S_0'\sim d_\pi$ and transition kernel 
$P_\pi$, coupled to $\{S_t\}$ such that $S_t = S_t'$ for all 
$t\ge\tau_{\mathrm{coal}}$ and 
$\Pr(\tau_{\mathrm{coal}}>t)
\le C_{\tref{lem:doeblin-coupling},1}\,
C_{\tref{lem:doeblin-coupling},2}^{\,t}$.
Define $M_t'$ from $\{S_t'\}$ analogously to $M_t$.
For $T_1$, since $M_t = M_t'$ for all $t\ge\tau_{\mathrm{coal}}$ 
and $\|M_t - M_t'\|_2 \le 2(1+\gamma)C_x^2$, we have
\[
\|T_1\|_2 \le \frac{2(1+\gamma)C_x^2\,\tau_{\mathrm{coal}}}{n}.
\]
Choosing 
\[
t_0 \doteq \biggl\lceil \frac{1}{\log(1/C_{\tref{lem:doeblin-coupling},2})}
\log\frac{3C_{\tref{lem:doeblin-coupling},1}\,n}{\delta}
\biggr\rceil,
\]
and defining
\begin{equation}
\label{eq:C14-def}
C_{\tref{lem:concentration},4}
\doteq 2(1+\gamma)C_x^2
\left(
\frac{1+\log(3C_{\tref{lem:doeblin-coupling},1})}{\log(1/C_{\tref{lem:doeblin-coupling},2})}+1
\right),
\end{equation}
by Lemma~\tref{lem:doeblin-coupling} we have 
$\Pr(\tau_{\mathrm{coal}}>t_0)\le \delta/(3n) \le\delta/3$.
Moreover, under the standing condition 
$n_{\mathrm{eff}}\ge d+\log(4/\delta)$, we have $\log(n/\delta)\ge 1$,
and hence on this event
\begin{equation}
\label{eq:T1-raw}
\|T_1\|_2 
\le \frac{2(1+\gamma)C_x^2\,t_0}{n}
\le \frac{C_{\tref{lem:concentration},4}\log(n/\delta)}{n}.
\end{equation}
This term will be absorbed after the block length $m$ is chosen in \eqref{eq:m-choice} below.

For $T_2$, define $\widetilde M_t' \doteq M_t' - M_\pi$. 
Since $\{S_t'\}$ is stationary, $\E[\widetilde M_t'] = 0$ 
for all $t$. Since each $\widetilde M_t'$ depends only on 
$(S_t', S_{t+1}')$, the process $\{\widetilde M_t'\}$ inherits 
geometric $\beta$-mixing from $\{S_t'\}$ via 
Lemma~\tref{lem:beta}, with mixing rate $\beta(m-1)$.
Moreover,
\begin{equation}
\label{eq:Mtilde-bound}
\|\widetilde M_t'\|_2 \le 2(1+\gamma)C_x^2.
\end{equation}
The blocking, Berbee coupling 
(Lemma~\tref{lem:berbee-blocks}), and matrix Bernstein
(Lemma~\tref{lem:berstein}) argument below applies to it
directly.

\paragraph{Uniform bounds and tail.}
Fix an arbitrary block length $m\in\mathbb{N}$ and let $K\doteq \lfloor n/m\rfloor$.
Define $n_{\mathrm{eff}} \doteq \lfloor K/2 \rfloor = \lfloor n/(2m) \rfloor$.
We use the first $2n_{\mathrm{eff}}$ full blocks and absorb all remaining terms into the tail. 
Specifically, partition the first $2mn_{\mathrm{eff}}$ indices into consecutive blocks 
$I_j \doteq \{jm,\dots,(j+1)m-1\}$ for $j=0,\dots,2n_{\mathrm{eff}}-1$, and define
\[
\widetilde M'^{(j)} \doteq \sum_{t\in I_j}\widetilde M'_t,
\qquad
R'_{\mathrm{tail}} \doteq \frac{1}{n}\sum_{t=2mn_{\mathrm{eff}}}^{n-1}\widetilde M'_t.
\]
Then
\begin{equation}
\label{eq:boundDeltaA}
    \frac{1}{n}\sum_{t=0}^{n-1}\widetilde M'_t 
    = \frac{1}{n}\sum_{j=0}^{2n_{\mathrm{eff}}-1}\widetilde M'^{(j)} + R'_{\mathrm{tail}}.
\end{equation}
Since the tail contains fewer than $2m$ terms, \eqref{eq:Mtilde-bound} implies
\begin{equation}
\label{eq:tail-bound}
\|R'_{\mathrm{tail}}\|_2 \le \frac{4m(1+\gamma)C_x^2}{n}.
\end{equation}
To apply Berbee coupling with valid separation, we split the $2n_{\mathrm{eff}}$ blocks 
into even-indexed and odd-indexed groups. We bound the even-block sum below, and the identical argument applies to the 
odd-block sum.

\paragraph{Berbee coupling for even blocks.}
Let $c_\beta\ge 1$ and $\rho\in(0,1)$ denote the geometric $\beta$-mixing constants from Lemma~\tref{lem:beta}, so that $\beta(q)\le c_\beta\rho^q$ for all $q\ge 1$.
Within the even subsequence, consecutive blocks $I_{2j}$ and $I_{2(j+1)}$ 
are separated by one intervening block. Since each $\widetilde M_t'$ depends on 
$(S_t',S_{t+1}')$, the effective separation in the original $S'$-chain is $m-1$.
Since $\{\widetilde M'_t\}$ is stationary and geometrically $\beta$-mixing 
by Lemma~\tref{lem:beta}, applying Lemma~\tref{lem:berbee} with this separation to the even block sums
$\{\widetilde M'^{(2j)}\}_{j=0}^{n_{\mathrm{eff}}-1}$,
there exist independent random matrices 
$\{\widetilde M^{*(2j)}\}_{j=0}^{n_{\mathrm{eff}}-1}$
with the same marginals as 
$\{\widetilde M'^{(2j)}\}_{j=0}^{n_{\mathrm{eff}}-1}$ such that, for the event
\begin{equation}
\label{eq:event}
    \mathcal E \doteq \Bigl\{\widetilde M'^{(2j)}=\widetilde M^{*(2j)}
    \ \text{for all}\ j=0,\dots,n_{\mathrm{eff}}-1\Bigr\},
\end{equation}
we have
\begin{equation}
\label{eq:berbee-fail}
\Pr(\mathcal E^c) \leq n_{\mathrm{eff}}\,\beta(m-1) 
\leq \frac{n}{2m}\,c_\beta\,\rho^{m-1}.
\end{equation}
Here $\mathcal E^c$ denotes the complement of the event $\mathcal E$.
Moreover, by \eqref{eq:Mtilde-bound}, almost surely
\begin{equation}
\label{eq:block-uniform}
\|\widetilde M^{*(2j)}\|_2
\le \sum_{t\in I_{2j}}\|\widetilde M'_t\|_2
\le 2m(1+\gamma)C_x^2.
\end{equation}

\paragraph{Matrix Bernstein on independent coupled blocks.}
Let $\overline M \doteq \sum_{j=0}^{n_{\mathrm{eff}}-1}\widetilde M^{*(2j)}$ and
\[
V_A \doteq \left\|\sum_{j=0}^{n_{\mathrm{eff}}-1}\E\bigl[\mathcal H(\widetilde M^{*(2j)})^2\bigr]\right\|_2,
\]
where $\mathcal H(\cdot)$ is the Hermitian dilation in Definition~\ref{def:hermitian}.
Using \eqref{eq:block-uniform}, the triangle inequality, Jensen's inequality, and $\|\mathcal H(M)\|_2=\|M\|_2$,
we obtain
\begin{align}
\label{eq:VA}
V_A
&\le \sum_{j=0}^{n_{\mathrm{eff}}-1}\E\!\left[\|\widetilde M^{*(2j)}\|_2^2\right]
\le 4n_{\mathrm{eff}}m^2(1+\gamma)^2C_x^4.
\end{align}
By Lemma~\ref{lem:berstein} with $(d_1,d_2)=(d,d)$, $N=n_{\mathrm{eff}}$,
$Z_j=\widetilde M^{*(2j)}$, and
$R=2m(1+\gamma)C_x^2$, for all $t>0$,
\begin{equation}
\label{eq:bernstein}
\Pr\Bigl(\|\overline M\|_2\ge t\Bigr)
\le
2d\exp\!\left(
-\frac{t^2/2}{V_A+\bigl(2m(1+\gamma)C_x^2\bigr)t/3}
\right).
\end{equation}
Let
\[
s_\delta \doteq \log\frac{24d}{\delta}
\]
and choose
\begin{equation}
\label{eq:t-choice}
t \doteq 
\sqrt{2V_A s_\delta+\frac{4m^2(1+\gamma)^2C_x^4s_\delta^2}{9}}
+\frac{2m(1+\gamma)C_x^2s_\delta}{3},
\end{equation}
so that
\begin{equation}
\label{eq:bernstein-delta2}
\Pr\Bigl(\|\overline M\|_2\ge t\Bigr)\le \frac{\delta}{12}.
\end{equation}

\paragraph{Choosing $m$ and concluding the bound for $\|T_2\|_2$.}
Set
\begin{equation}
\label{eq:m-choice}
m \doteq \left\lceil C_{\tref{lem:concentration},1}\log\frac{C_{\tref{lem:concentration},2}n}{\delta}\right\rceil,
\qquad
C_{\tref{lem:concentration},1}\doteq \frac{1}{\log(1/\rho)},
\qquad
C_{\tref{lem:concentration},2}
\doteq \max\left\{\frac{3c_\beta}{\rho},\,\rho^{-2}\right\}.
\end{equation}
Then $m\ge 2$ because $C_{\tref{lem:concentration},2}\ge \rho^{-2}$, and
\begin{equation}
\rho^{m-1}
\le \rho^{-1}\exp\!\left(-m\log(1/\rho)\right)
\le \frac{\delta}{\rho C_{\tref{lem:concentration},2}n}
\le \frac{\delta}{3c_\beta n}.
\end{equation}
Substituting into \eqref{eq:berbee-fail} yields
\begin{equation}
\label{eq:Ecbound}
\Pr(\mathcal E^c)
\le \frac{n}{2m}c_\beta\rho^{m-1}
\le \frac{n}{2m}c_\beta\cdot \frac{\delta}{3c_\beta n}
= \frac{\delta}{6m}
\le \frac{\delta}{12},
\end{equation}
where the last inequality uses $m\ge 2$.

Applying the identical argument to the odd-indexed blocks, the same bound holds. Combining Berbee's coupling and Bernstein's inequality for the even and odd blocks, with probability at least 
$1-\delta/3$,
\[
\|T_2\|_2 
\le \frac{2t}{n} + \|R'_{\mathrm{tail}}\|_2.
\]
Substituting \eqref{eq:tail-bound}, \eqref{eq:t-choice}, and \eqref{eq:VA} gives that on this event,
\begin{align}
\|T_2\|_2
&\le
\frac{2}{n}
\left(
\sqrt{2V_A s_\delta+\frac{4m^2(1+\gamma)^2C_x^4s_\delta^2}{9}}
+\frac{2m(1+\gamma)C_x^2s_\delta}{3}
\right)
+\frac{4m(1+\gamma)C_x^2}{n} \notag\\
&\le
\frac{4m(1+\gamma)C_x^2}{n}
\left(
\sqrt{2n_{\mathrm{eff}}s_\delta+\frac{s_\delta^2}{9}}
+\frac{s_\delta}{3}+1
\right).
\label{eq:deltaA-pre}
\end{align}
Define the universal constants
\begin{equation}
\label{eq:C16C17-def}
C_{\tref{lem:concentration},6}
\doteq 1+\frac{\log 6-1}{1+\log 4},
\qquad
C_{\tref{lem:concentration},7}
\doteq
\sqrt{C_{\tref{lem:concentration},6}
\left(2+\frac{C_{\tref{lem:concentration},6}}{9}\right)}
+\frac{C_{\tref{lem:concentration},6}}{3}+1.
\end{equation}
Since $\log d\le d-1$ for $d\ge 1$ and
$d+\log(4/\delta)\ge 1+\log 4$, we have
\[
s_\delta=\log\frac{24d}{\delta}
=\log d+\log\frac{4}{\delta}+\log 6
\le d+\log\frac{4}{\delta}+\log 6-1
\le C_{\tref{lem:concentration},6}\left(d+\log\frac{4}{\delta}\right).
\]
Using $n\ge 2mn_{\mathrm{eff}}$, the preceding bound on $s_\delta$, and 
$n_{\mathrm{eff}}\ge d+\log(4/\delta)$, \eqref{eq:deltaA-pre} implies
\begin{align}
\|T_2\|_2
&\le
2(1+\gamma)C_x^2 C_{\tref{lem:concentration},7}
\sqrt{\frac{d+\log(4/\delta)}{n_{\mathrm{eff}}}}.
\label{eq:deltaA-neff}
\end{align}
The choice of $m$ in \eqref{eq:m-choice} also gives 
$\log(n/\delta)\le \log(C_{\tref{lem:concentration},2}n/\delta)
\le m/C_{\tref{lem:concentration},1}$.
Therefore, \eqref{eq:T1-raw} implies that on the coupling event,
\begin{align}
\|T_1\|_2
&\le \frac{C_{\tref{lem:concentration},4}}{2C_{\tref{lem:concentration},1}}
\cdot\frac{1}{n_{\mathrm{eff}}} \notag\\
&\le \frac{C_{\tref{lem:concentration},4}}{2C_{\tref{lem:concentration},1}}
\sqrt{\frac{d+\log(4/\delta)}{n_{\mathrm{eff}}}},
\label{eq:T1-absorb}
\end{align}
where the first inequality uses $n\ge 2mn_{\mathrm{eff}}$ and the second uses $n_{\mathrm{eff}}\ge d+\log(4/\delta)$.
Combining \eqref{eq:T1-absorb} and \eqref{eq:deltaA-neff}, on the intersection of the coupling event and the $T_2$ block-concentration event,
\begin{equation}
\label{eq:epsilonA-final-preC13}
\varepsilon_A=\|\Delta_A\|_2
\le \left(
\frac{C_{\tref{lem:concentration},4}}{2C_{\tref{lem:concentration},1}}
+2(1+\gamma)C_x^2 C_{\tref{lem:concentration},7}
\right)
\sqrt{\frac{d+\log(4/\delta)}{n_{\mathrm{eff}}}}.
\end{equation}

\paragraph{Bounding $\varepsilon_b$.}
Recall $\widehat b=\frac1n\sum_{t=0}^{n-1}R_{t+1}x_t$ and define
$m_t \doteq R_{t+1}x_t$ and $m_\pi \doteq \E_{d_\pi}[m_0] = b$.
Define $m_t'$ from $\{S_t'\}$ analogously to $m_t$.
Then
\begin{equation}
\label{eq:Deltab}
    \Delta_b \doteq \widehat b-b
= \underbrace{\frac{1}{n}\sum_{t=0}^{n-1}(m_t - m_t')}_{T_3}
+ \underbrace{\frac{1}{n}\sum_{t=0}^{n-1}(m_t' - m_\pi)}_{T_4}.
\end{equation}
By the same coupling argument as for $T_1$, since 
$m_t = m_t'$ for all $t \ge \tau_{\mathrm{coal}}$ and 
$\|m_t - m_t'\|_2 \le 2C_RC_x$, on the same coupling event we have
\[
\|T_3\|_2 \le \frac{2C_RC_x\,t_0}{n}
\le \frac{C_{\tref{lem:concentration},5}}{2C_{\tref{lem:concentration},1}}
\sqrt{\frac{d+\log(4/\delta)}{n_{\mathrm{eff}}}},
\]
where
\begin{equation}
\label{eq:C15-def}
C_{\tref{lem:concentration},5}
\doteq 2C_RC_x
\left(
\frac{1+\log(3C_{\tref{lem:doeblin-coupling},1})}{\log(1/C_{\tref{lem:doeblin-coupling},2})}+1
\right),
\end{equation}
and the absorption uses the same steps as \eqref{eq:T1-absorb}.
For $T_4$, define $\widetilde m_t' \doteq m_t' - m_\pi$. 
Since $\{S_t'\}$ is stationary, $\E[\widetilde m_t'] = 0$ and 
$\|\widetilde m_t'\|_2 \le 2C_RC_x$.
Applying the same blocking, Berbee coupling, and Bernstein argument as above (treating vectors as $d\times 1$ matrices and using $d+1\le 2d$), with the same choices of $m$ and $n_{\mathrm{eff}}$, yields that with probability at least $1-\delta/3$,
\[
\|T_4\|_2
\le 2C_RC_x C_{\tref{lem:concentration},7}
\sqrt{\frac{d+\log(4/\delta)}{n_{\mathrm{eff}}}}.
\]
Therefore, on the intersection of the coupling event and the $T_4$ block-concentration event,
\begin{equation}
\label{eq:epsilonb-final-preC13}
\varepsilon_b=\|\Delta_b\|_2
\le
\left(
\frac{C_{\tref{lem:concentration},5}}{2C_{\tref{lem:concentration},1}}
+2C_RC_x C_{\tref{lem:concentration},7}
\right)
\sqrt{\frac{d+\log(4/\delta)}{n_{\mathrm{eff}}}}.
\end{equation}

\paragraph{Union bound.}
Finally, let
\begin{equation}
\label{eq:C13-final-def}
C_{\tref{lem:concentration},3}
\doteq
\max\left\{
\frac{C_{\tref{lem:concentration},4}}{2C_{\tref{lem:concentration},1}}
+2(1+\gamma)C_x^2 C_{\tref{lem:concentration},7},
\frac{C_{\tref{lem:concentration},5}}{2C_{\tref{lem:concentration},1}}
+2C_RC_x C_{\tref{lem:concentration},7}
\right\}.
\end{equation}
The coupling event fails with probability at most $\delta/3$, the $T_2$ block-concentration event fails with probability at most $\delta/3$, and the $T_4$ block-concentration event fails with probability at most $\delta/3$. Hence a union bound gives
\[
\Pr\Bigl(\varepsilon_A\le C_{\tref{lem:concentration},3}
\sqrt{\frac{d+\log(4/\delta)}{n_{\mathrm{eff}}}}
\ \text{and}\ 
\varepsilon_b\le C_{\tref{lem:concentration},3}
\sqrt{\frac{d+\log(4/\delta)}{n_{\mathrm{eff}}}}\Bigr)
\ge 1-\delta,
\]
which completes the proof.
\end{proof}

\subsection{Proof of Theorem~\ref{thm:finite}}
\label{proof:finite}
\begin{proof}
On $\mathcal E_\delta$, by definition in \eqref{eq:DeltaA} and \eqref{eq:Deltab}, Lemma~\ref{lem:concentration} implies
\begin{equation}
\label{eq:thm-finite-DeltaA}
\|\Delta_A\|_2 \le \varepsilon_A(\delta),
\qquad
\|\Delta_b\|_2 \le \varepsilon_b(\delta).
\end{equation}
Let $e_k \doteq w_k - w^*$.
From \eqref{eq:empirical-recursion}, subtracting $w^*$ from both sides gives
\begin{equation}
\label{eq:thm-finite-e-rec}
e_{k+1} = (I - \alpha A)e_k + \alpha(\Delta_b - \Delta_A w_k).
\end{equation}
Let $g_k \doteq \Delta_b - \Delta_A w_k$.
Expanding \eqref{eq:thm-finite-e-rec} gives
\begin{equation}
\label{eq:thm-finite-expand}
\|e_{k+1}\|_2^2
=
\underbrace{\|(I-\alpha A)e_k\|_2^2}_{T_1}
+\underbrace{2\alpha\abs{\langle (I-\alpha A)e_k,\, g_k\rangle}}_{T_2}
+\underbrace{\alpha^2\|g_k\|_2^2}_{T_3}.
\end{equation}
We bound the terms one by one.
For $T_1$, since $\alpha\in(0,\mu/L]$, by the same argument as \eqref{eq:one_step_contract},
\begin{equation}
\label{eq:T1}
T_1 \le (1-\alpha\mu)\|e_k\|_2^2,
\end{equation}
which also implies $\|(I-\alpha A)e_k\|_2^2 \le \|e_k\|_2^2$ since $\alpha\mu > 0$.

For $T_2$, we have
\begin{align}
\label{eq:thm-finite-cross}
T_2 &\le 2\alpha\|(I-\alpha A)e_k\|_2\|g_k\|_2 \notag\\
&\le 2\alpha\|e_k\|_2\|g_k\|_2 \quad\text{(Cauchy--Schwarz)}\notag\\
&\le \frac{\alpha\mu}{4}\|e_k\|_2^2 + \frac{4\alpha}{\mu}\|g_k\|_2^2.
\quad\text{(AM--GM)}
\end{align}

For $T_3$, since $\|w_k\|_2 \le \|w^*\|_2 + \|e_k\|_2$, we have
\begin{align}
\label{eq:thm-finite-g}
\|g_k\|_2^2
&\le 2\bigl(\|\Delta_b\|_2^2 + \|\Delta_A\|_2^2\|w_k\|_2^2\bigr) \notag\\
&\le 2\bigl(\varepsilon_b(\delta)^2
+ 2\varepsilon_A(\delta)^2(\|w^*\|_2^2 + \|e_k\|_2^2)\bigr).
\end{align}

Combining the above inequalities yields
\begin{align}
\label{eq:thm-finite-basic}
\|e_{k+1}\|_2^2
&\le
\Bigl(1-\frac{3\alpha\mu}{4}\Bigr)\|e_k\|_2^2
+ 2\alpha\Bigl(\frac{4}{\mu}+\alpha\Bigr)
\bigl(\varepsilon_b(\delta)^2
+ 2\varepsilon_A(\delta)^2(\|w^*\|_2^2 + \|e_k\|_2^2)\bigr)\notag\\
&\le
\Bigl(1-\frac{3\alpha\mu}{4}
+ C_{\text{Thm}\tref{thm:finite},3}\varepsilon_A(\delta)^2\Bigr)\|e_k\|_2^2
+ C_{\text{Thm}\tref{thm:finite},4}
\bigl(\varepsilon_b(\delta)^2+\varepsilon_A(\delta)^2\bigr),
\end{align}
where
$C_{\text{Thm}\tref{thm:finite},3}
\doteq 4\alpha\bigl(\frac{4}{\mu}+\alpha\bigr)$,\quad
$C_{\text{Thm}\tref{thm:finite},4}
\doteq 2\alpha\bigl(\frac{4}{\mu}+\alpha\bigr)(1+2\|w^*\|_2^2)$.

By Lemma~\ref{lem:concentration},
\begin{equation}
\label{eq:thm-finite-epsA-rate}
\varepsilon_A(\delta)^2
\le C_{\tref{lem:concentration},3}^2
\frac{d+\log(4/\delta)}{n_{\mathrm{eff}}}.
\end{equation}
Let $C_{\text{Thm}\tref{thm:finite},1}
\doteq \frac{4C_{\text{Thm}\tref{thm:finite},3}C_{\tref{lem:concentration},3}^2}{\alpha\mu}$.
Under $n_{\mathrm{eff}}\ge C_{\text{Thm}\tref{thm:finite},1}(d+\log(4/\delta))$, it follows that
\begin{equation}
\label{eq:thm-finite-absorb}
C_{\text{Thm}\tref{thm:finite},3}\varepsilon_A(\delta)^2
\le \frac{\alpha\mu}{4}.
\end{equation}
Plugging \eqref{eq:thm-finite-absorb} into \eqref{eq:thm-finite-basic} yields
\begin{equation}
\label{eq:thm-finite-rec}
\|e_{k+1}\|_2^2
\le
\Bigl(1-\frac{\alpha\mu}{2}\Bigr)\|e_k\|_2^2
+ C_{\text{Thm}\tref{thm:finite},4}\varepsilon(\delta)^2.
\end{equation}
Iterating \eqref{eq:thm-finite-rec} yields
\begin{equation}
\label{eq:thm-finite-unroll}
\|e_k\|_2^2
\le
\Bigl(1-\frac{\alpha\mu}{2}\Bigr)^k\|e_0\|_2^2
+ \frac{2C_{\text{Thm}\tref{thm:finite},4}}{\alpha\mu}\varepsilon(\delta)^2.
\end{equation}
Since $\fL(w_k) = \|C^{-\frac12}Ae_k\|_2^2 \le \|C^{-\frac12}A\|_2^2\|e_k\|_2^2$
and $\|e_0\|_2^2 \le \fL(w_0)/\sigma_{\min}(C^{-\frac12}A)^2$, we obtain
\[
\fL(w_k) \le C_{\text{Thm}\tref{thm:finite},2}
\Bigl(\Bigl(1-\frac{\alpha\mu}{2}\Bigr)^k\fL(w_0)
+ \varepsilon(\delta)^2\Bigr),
\]
where $C_{\text{Thm}\tref{thm:finite},2} \doteq
\max\!\bigl(\|C^{-\frac12}A\|_2^2/\sigma_{\min}(C^{-\frac12}A)^2,\;
2C_{\text{Thm}\tref{thm:finite},4}/(\alpha\mu)\bigr)$,
completing the proof.
\end{proof}

\subsection{Proof of Corollary~\tref{cor:floor-depth}}
\label{proof:floor-depth}
\begin{proof}
Given \eqref{eq:finite-main},
since $\alpha\mu \in (0,1]$, we can use $\log(1-x)\le -x$ for $x\in(0,1)$ to obtain
\begin{equation}
\left(1-\frac{\alpha\mu}{2}\right)^k
=
\exp\left(k\log\left(1-\frac{\alpha\mu}{2}\right)\right)
\le
\exp\left(-\frac{\alpha\mu}{2}k\right).
\end{equation}
If $\mathcal L(w_0)\le \varepsilon(\delta)^2$, then $\mathcal L(w_k)\le 2C_{\text{Thm}\tref{thm:finite},2}\varepsilon(\delta)^2$ holds trivially.
Otherwise, whenever
\begin{equation}
\label{eq:cor-floor-k}
k
\ge
\frac{2}{\alpha\mu}
\log\left(
\frac{\mathcal{L}(w_0)}{\varepsilon(\delta)^2}
\right),
\end{equation}
we have
\begin{equation}
\left(1-\frac{\alpha\mu}{2}\right)^k \mathcal{L}(w_0)
\le
\varepsilon(\delta)^2.
\end{equation}
Plugging this bound into \eqref{eq:finite-main} yields
\begin{equation}
\mathcal{L}(w_k)
\le
2C_{\text{Thm}\tref{thm:finite},2}\varepsilon(\delta)^2,
\end{equation}
which completes the proof.
\end{proof}


\section{Proofs in Section~\tref{sec:emergence}}
\label{sec:proofsec6}
This section provides the proofs of Theorem~\tref{thm:emergence} and Corollary~\tref{cor:bridge}, together with the technical lemmas they rely on. We first introduce auxiliary quantities and a factorization of $\Delta_k$ that will be used throughout.

\paragraph{Factorization of $\Delta_k$.}
For any $\theta\in\Theta^{\mathrm{CoT}}$ and trajectory $\tau_n$, define the Jacobian factor
\begin{equation}
\label{eq:def-Gk}
G_k(\theta;\tau_n)
\doteq
\nabla_\theta w_k(\theta;\tau_n)\in\R[3d^2\times d].
\end{equation}
The chain rule gives, for all $j$,
\begin{equation}
\label{eq:grad-vk-Gk}
\nabla_\theta \hat v_k(S_j;\theta,\tau_n)
=
G_k(\theta;\tau_n) x_j.
\end{equation}
Using \eqref{eq:grad-vk-Gk} and the empirical moments $(\widehat A,\widehat b)$ in \eqref{eq:empirical-moments}, computed from $\tau_n$, we obtain the exact factorization
\begin{equation}
\label{eq:neu-Gk-res}
\Delta_k(\theta;\tau_n)
=
G_k(\theta;\tau_n)
\big(\widehat b-\widehat A w_k(\theta;\tau_n)\big).
\end{equation}

\paragraph{Empirical MSPBE.}
For any $w\in\mathbb{R}^d$, define the empirical MSPBE as
\begin{equation}
\label{eq:def-Lhat}
\widehat{\mathcal L}(w)
\doteq
\big\|C^{-\frac12}\big(\widehat b-\widehat A w\big)\big\|_2^2.
\end{equation}
Plugging \eqref{eq:neu-Gk-res} and \eqref{eq:def-Lhat} into \eqref{eq:EUN}, we obtain
\begin{equation}
\label{eq:neu-bound-local}
J_k(\theta;\tau_n)
\le
\bigl\|G_k(\theta;\tau_n) C^{\frac12}\bigr\|_{2}
\sqrt{\widehat{\mathcal L}\big(w_k(\theta;\tau_n)\big)}.
\end{equation}
Applying \eqref{eq:neu-bound-local} to each trajectory in $\mathcal D$ and averaging gives
\begin{equation}
\label{eq:dataset-two-terms}
J_k(\theta_*;\mathcal D)
\le\frac1{N_{\mathrm{train}}}\sum_{\tau_n\in\mathcal D}
\norm{G_k(\theta_*;\tau_n) C^{\frac12}}_{2}
\sqrt{\widehat{\mathcal L}\qty(w_k(\theta_*;\tau_n))}.
\end{equation}
Therefore, it suffices to control the two factors on the right-hand side of \eqref{eq:dataset-two-terms} for each trajectory $\tau_n$.

\paragraph{Empirical curvature constants.}
Define $\widehat\mu$ and $\widehat L$ by taking \eqref{eq:muL} and replacing $(A,b)$ with $(\widehat A,\widehat b)$. The remaining analysis relies on well-conditioned empirical geometry, which requires $\widehat\mu>0$. Lemma~\tref{lem:empirical-mu} establishes a sample-size threshold under which $\widehat\mu>0$ on $\mathcal E_\delta$.

\paragraph{Dataset-level event and constants.}
We apply Lemma~\tref{lem:concentration} to each trajectory $\tau_n^{(i)} \in \mathcal D$ with $\delta' \doteq \delta/N_{\mathrm{train}}$. Let $\mathcal E_{\delta'}^{(i)}$ denote the corresponding event for $\tau_n^{(i)}$, and define the dataset-level event
\begin{equation}
\label{eq:eventD}
\mathcal E_{\delta}(\mathcal D)
\doteq
\bigcap_{i\in[N_{\mathrm{train}}]} \mathcal E_{\delta'}^{(i)}.
\end{equation}
By a union bound, $\mathbb P\!\big(\mathcal E_{\delta}(\mathcal D)\big)\ge 1-\delta$. We further define the dataset-level parameters
\[
\widehat\mu_{\mathcal D} \doteq \min_{i\in[N_{\mathrm{train}}]} \widehat\mu^{(i)},
\quad
\widehat L_{\mathcal D} \doteq \max_{i\in[N_{\mathrm{train}}]} \widehat L^{(i)}.
\]
\subsection{Proof of Lemma~\tref{lem:empirical-mu}}
\label{proof:empirical-mu}
\begin{lemma}
\label{lem:empirical-mu}
Fix $\delta\in(0,1)$. There exists a constant $C_\tref{lem:empirical-mu}$ such that on $\mathcal E_\delta$, if
\begin{equation}
\label{eq:neff-threshold-mu}
n_{\mathrm{eff}}
\ \ge\
C_\tref{lem:empirical-mu}
\Bigl(d+\log\frac{4}{\delta}\Bigr),
\end{equation}
then $\widehat\mu$ satisfies
\begin{equation}
\label{eq:muhat-lb}
\widehat\mu\ge\frac{\mu}{2} >0,
\quad \widehat L \leq 4L.
\end{equation}
\end{lemma}
\begin{proof}
Recall the population quantities
\[
H \doteq \tfrac12(A+A^\top),
\quad
\mu \doteq \lambda_{\min}(H).
\]
Define the empirical counterparts by replacing $A$ with $\widehat A$,
\begin{equation}
\label{eq:empHmu}
    \widehat H \doteq \tfrac12(\widehat A+\widehat A^\top),
\quad
\widehat\mu \doteq \lambda_{\min}(\widehat H),
\quad
\widehat L \doteq \|\widehat A\|_2^2.
\end{equation}
We apply Weyl's inequality for symmetric matrices and get
\begin{align}
    \widehat\mu
&= \lambda_{\min}\bigl(H + (\widehat H - H)\bigr)\notag\\
&\ge \lambda_{\min}(H) - \|\widehat H - H\|_2\notag\\
&= \mu - \Bigl\|\tfrac12\bigl((\widehat A-A)+(\widehat A-A)^\top\bigr)\Bigr\|_2\notag\\
&\ge \mu - \|\widehat A - A\|_2\notag\\
&= \mu - \varepsilon_A \quad\text{(By \eqref{eq:varepsilon_ab})}\notag\\
&\ge \mu - C_{\tref{lem:concentration},3}\sqrt{\frac{d+\log(4/\delta)}{n_{\mathrm{eff}}}}
\quad \text{(By Lemma~\ref{lem:concentration})}.
\end{align}
Thus, whenever
$n_{\mathrm{eff}} \ge C_{\tref{lem:empirical-mu},1}\bigl(d+\log(4/\delta)\bigr)$,
we obtain $\widehat\mu \ge \frac{\mu}{2} > 0$.

For $\widehat L$, by the triangle inequality, if $n_{\mathrm{eff}} \geq \frac{C_{\tref{lem:concentration},3}^2}{L}\bigl(d+\log(4/\delta)\bigr)$, we have
\begin{align}
    \|\widehat A\|_2 &\leq \|A\|_2 + \varepsilon_A \\
    &\leq \sqrt{L} + C_{\tref{lem:concentration},3}\sqrt{\frac{d+\log(4/\delta)}{n_{\mathrm{eff}}}} \\
    &\leq 2\sqrt{L},
\end{align}
Hence $\widehat L = \|\widehat A\|_2^2 \leq 4L$.
Thus, choosing
$C_\tref{lem:empirical-mu} \doteq \max\!\left(\frac{4\,C_{\tref{lem:concentration},3}^2}{\mu^2},\;\frac{C_{\tref{lem:concentration},3}^2}{L}\right)$
completes the proof.

\end{proof}

\subsection{Proof of Lemma~\tref{lem:residual-Lhat}}
\label{proof:Lhat-closed}
\begin{lemma}
\label{lem:residual-Lhat}
For any trajectory $\tau_n$ and stepsize $\alpha \in (0,\mu/(8L)]$. On $\mathcal E_\delta$, if \eqref{eq:neff-threshold-mu} holds, there exists a constant $C_\tref{lem:residual-Lhat}$ such that for all $k\geq 0$,
\begin{equation}
\label{eq:empirical_contraction}
\widehat{\mathcal L}(w_k\qty(\theta_*;\tau_n)) \le C_\tref{lem:residual-Lhat}(1 - \alpha \widehat{\mu})^k \widehat{\mathcal L}(w_0\qty(\theta_*; \tau_n)).
\end{equation}
\end{lemma}
\begin{proof}
Since $\alpha \leq \mu/(8L)$ and by Lemma~\ref{lem:empirical-mu}, 
$\widehat\mu \geq \mu/2$ and $\widehat L \leq 4L$ on $\cE_\delta$, 
we have $\alpha \leq \widehat\mu/\widehat L$.
Fix a trajectory $\tau_n$ and work throughout this proof with the empirical moments
$(\widehat A,\widehat b)$ computed from $\tau_n$.
We also fix $\theta=\theta_*$, and for brevity write
\[
w_k \doteq w_k(\theta_*;\tau_n),\qquad k\ge 0.
\]
With \eqref{eq:neff-threshold-mu}, we have $\widehat\mu>0$ by Lemma~\tref{lem:empirical-mu}.
Let $\widehat w^*$ denote the unique solution to $\widehat A w = \widehat b$.
Define $\widehat e_k \doteq w_k - \widehat w^*$.
By the empirical TD update in \eqref{eq:empirical-recursion}, subtracting $\widehat w^*$ from both sides gives
\[
\widehat e_{k+1} = (I - \alpha \widehat A)\widehat e_k.
\]
By the same argument as \eqref{eq:one_step_contract} with $A$ replaced by $\widehat A$,
for $\alpha \in (0, \widehat\mu/\widehat L]$ we have
\begin{equation}
\label{eq:ehat-contract}
\|\widehat e_{k+1}\|_2^2 \le (1 - \alpha\widehat\mu)\|\widehat e_k\|_2^2.
\end{equation}
Iterating over $k$ gives
\[
\|\widehat e_k\|_2^2 \le (1 - \alpha\widehat\mu)^k \|\widehat e_0\|_2^2.
\]
Since $\widehat{\fL}(w_k) = \|C^{-\frac12}\widehat A\, \widehat e_k\|_2^2$,
by the same argument as in the proof of Proposition~\tref{prop:expect},
\[
\widehat{\fL}(w_k) \le C_\tref{lem:residual-Lhat}(1 - \alpha\widehat\mu)^k\, \widehat{\fL}(w_0),
\]
where $C_\tref{lem:residual-Lhat} \doteq \|C^{-\frac12}\widehat A\|_2^2 / \sigma_{\min}(C^{-\frac12}\widehat A)^2$,
which completes the proof.
\end{proof}

\begin{lemma}[Jacobian of the block-sparse parameterization]
\label{lem:pqderiv}
Under the block-sparse parameterization \eqref{eq:Theta-CoT}, we identify
$\theta\in\R[3d^2]$ with the free blocks $(P_1,Q_1,Q_2)$.
Let $D_\theta P$ and $D_\theta Q$ denote the Fr\'echet derivatives with respect to $\theta$.
For any direction $u\in\R[3d^2]$, define
\begin{equation}
\label{eq:deltapq}
    \Delta P(u)\doteq D_\theta P[u],\qquad \Delta Q(u)\doteq D_\theta Q[u].
\end{equation}
Then for all $\|u\|_2=1$, $\|\Delta P(u)\|_2\le 1$ and $\|\Delta Q(u)\|_2\le 1$.
\end{lemma}

\begin{proof}
Under \eqref{eq:Theta-CoT}, the mapping from the free blocks $(P_1,Q_1,Q_2)$ to $(P,Q)$ is linear.
Let $\mathrm{unvec}:\R[d^2]\to\R[d\times d]$ denote the inverse of $\mathrm{vec}(\cdot)$,
i.e., $\mathrm{unvec}(x)$ reshapes a vector $x\in\R[d^2]$ into a $d\times d$ matrix.
Write $u=(u_P,u_{Q_1},u_{Q_2})$ with each component in $\R[d^2]$, and define
\[
U_{P_1}\doteq \mathrm{unvec}(u_P),\quad
U_{Q_1}\doteq \mathrm{unvec}(u_{Q_1}),\quad
U_{Q_2}\doteq \mathrm{unvec}(u_{Q_2}).
\]
By \eqref{eq:Theta-CoT}, $P$ depends on $\theta$ only through the free block $P_1$, while all other blocks are fixed constants.
Hence $D_\theta P[u]$ vanishes on all fixed entries and equals $U_{P_1}$ on the free block.
Similarly, $D_\theta Q[u]$ vanishes on all fixed entries and equals $(U_{Q_1},U_{Q_2})$ on the two free blocks.
Therefore,
\[
\|\Delta P(u)\|_2 \le \|\Delta P(u)\|_F = \|U_{P_1}\|_F = \|u_P\|_2 \le \|u\|_2,
\]
and
\[
\|\Delta Q(u)\|_2 \le \|\Delta Q(u)\|_F
= \sqrt{\|U_{Q_1}\|_F^2+\|U_{Q_2}\|_F^2}
= \sqrt{\|u_{Q_1}\|_2^2+\|u_{Q_2}\|_2^2}
\le \|u\|_2.
\]
Taking $\|u\|_2=1$ yields the claim.
\end{proof}

\subsection{Proof of Lemma~\tref{lem:pqstarbound}}
\label{proof:pqstarbound}
\begin{lemma}[Norm of $P$]
\label{lem:pqstarbound}
For $P$ defined in \eqref{eq:para_construct}, we have $\norm{P}_2=\alpha$.
\end{lemma}

\begin{proof}
Under \eqref{eq:Theta-CoT}, the matrix $P$ has a single nonzero block $P_1\in\R[d\times d]$.
Hence,
\begin{equation}
\label{eq:Pnorm-eq-P1}
  \|P\|_2=\|P_1\|_2.
\end{equation}
Moreover, \eqref{eq:para_construct} gives $P_1=\alpha I_d$.
Therefore $\|P\|_2=\|\alpha I_d\|_2=\alpha$.
\end{proof}

\subsection{Proof of Lemma~\tref{lem:wbound}}
\label{proof:wbound}
\begin{lemma}[Uniform bound on the empirical iterates]
\label{lem:wbound}
On $\mathcal E_\delta$, if \eqref{eq:neff-threshold-mu} holds, then there exists a constant $C_\tref{lem:wbound}>0$ such that
for all $k\ge 0$, the iterates $\{w_k\}$ generated by \eqref{eq:empirical-recursion} satisfy
\[
\|w_k\|_2 \le C_\tref{lem:wbound}.
\]
\end{lemma}
\begin{proof}
Let $\widehat w_*$ be the unique solution to $\widehat A w=\widehat b$.
On $\mathcal E_\delta$, if \eqref{eq:neff-threshold-mu} holds, then by Lemma~\tref{lem:empirical-mu} we have $\widehat\mu \ge \frac{\mu}{2} > 0$.
Let $\widehat e_k \doteq w_k - \widehat w_*$.
By Lemma~\tref{lem:residual-Lhat}, $\|\widehat e_k\|_2 \le \|\widehat e_0\|_2$ for all $k \ge 0$.
Since $w_0 = \mathbf{0}$, we have $\widehat e_0 = -\widehat w_*$, and hence
\[
\sup_{k\ge 0}\|w_k\|_2 \le \|\widehat w_*\|_2 + \|\widehat e_0\|_2 = 2\|\widehat w_*\|_2.
\]
It remains to bound $\|\widehat w_*\|_2 = \|\widehat A^{-1}\widehat b\|_2$.
Since $\widehat\mu > 0$, $\widehat A$ is invertible and
$\|\widehat A^{-1}\|_2 \le \widehat\mu^{-1}$ (as $v^\top \widehat A v \ge \widehat\mu\|v\|_2^2$ for all $v$).
Therefore,
\begin{align}
\sup_{k\ge 0}\|w_k\|_2
&\le 2\widehat\mu^{-1}\|\widehat b\|_2\notag\\
&\le \frac{2}{\widehat\mu}\cdot\frac{1}{n}\sum_{t=0}^{n-1}|R_{t+1}|\|x_t\|_2\notag\\
&\le \frac{2C_R C_x}{\widehat\mu}.
\end{align}
Choosing $C_\tref{lem:wbound} \doteq \frac{2C_R C_x}{\widehat\mu}$ completes the proof.
\end{proof}

\subsection{Proof of Lemma~\tref{lem:sensitivity-Gk}}
\label{proof:sensitivity-Gk}
\begin{lemma}
\label{lem:sensitivity-Gk}
Fix any trajectory $\tau_n$ and stepsize $\alpha\in(0,\mu/(8L)]$. On $\mathcal E_\delta$, if \eqref{eq:neff-threshold-mu} holds, then there exists a constant $C_{\tref{lem:sensitivity-Gk}}\ge 0$ such that for all $k\ge 0$,
\begin{equation}
\label{eq:Gk-bound-uniform}
\bigl\|G_k(\theta_*;\tau_n) C^{\frac12}\bigr\|_2
\ \le\
C_{\tref{lem:sensitivity-Gk}}.
\end{equation}
\end{lemma}
\begin{proof}
Since $\alpha \leq \mu/(8L)$ and by Lemma~\ref{lem:empirical-mu}, 
$\widehat\mu \geq \mu/2$ and $\widehat L \leq 4L$ on $\cE_\delta$, 
we have $\alpha \leq \widehat\mu/\widehat L$.
Fix a trajectory $\tau_n$ and work throughout this proof with the empirical moments
$(\widehat A,\widehat b)$ computed from $\tau_n$.
We also fix $\theta=\theta_*$, and for brevity write $w_k \doteq w_k(\theta_*;\tau_n)$.
Thus, by \eqref{eq:def-Gk} we abbrieviate
\[
G_k \doteq \nabla_\theta w_k \in \R[3d^2\times d].
\]
Let $\mathcal F_\theta:\R[d]\to\R[d]$ denote the one-step map of our model with parameter $\theta$,
and write $F\doteq \mathcal F_{\theta_*}$.
In particular, under $\theta=\theta_*$ we have
\begin{equation}
\label{eq:w-rec-empirical}
w_{k+1}=F(w_k)=w_k+\alpha(\widehat b-\widehat A w_k).
\end{equation}
Differentiate $w_{k+1}(\theta)=\mathcal F_\theta(w_k(\theta))$ w.r.t.\ $\theta$
and evaluate at $\theta=\theta_*$ then gives
\begin{align}
\label{eq:G-rec-chain-new}
G_{k+1}
=G_k \nabla_w F(w_k) + \nabla_\theta F(w_k)
=G_k(I-\alpha\widehat A)+\nabla_\theta F(w_k).
\end{align}
Note that $G_0(\theta_*)=0$ since $w_0=\0$ in the prompt.
Right-multiplying by $C^{\frac12}$ and unrolling \eqref{eq:G-rec-chain-new} gives
\begin{align}
\label{eq:CJk-tri}
  \|G_kC^{1/2}\|_2
  &\le
  \sum_{i=0}^{k-1}
  \underbrace{\bigl\|C^{-\frac12}(I-\alpha\widehat A)^{k-1-i}C^{\frac12}\bigr\|_2}_{E_1}
  \underbrace{\|\nabla_\theta F(w_i) C^{\frac12}\|_2}_{E_2}.
\end{align}
We first bound $E_1$.
For any integer $m\ge 0$, by submultiplicativity,
\begin{equation}
\label{eq:prod-bound}
\norm{C^{-\frac12}(I-\alpha\widehat A)^m C^{\frac12}}_2
\leq \norm{C^{-\frac12}}_2 \norm{(I-\alpha\widehat A)^m}_2 \norm{C^{\frac12}}_2
\leq C_{\tref{lem:sensitivity-Gk},1}\, \norm{I-\alpha\widehat A}_2^{m},
\end{equation}
where $C_{\tref{lem:sensitivity-Gk},1} \doteq \sqrt{\lambda_{\max}(C)/\lambda_{\min}(C)}$.
Since $\widehat\mu>0$ by Lemma~\tref{lem:empirical-mu}, the same algebra as in \eqref{eq:one_step_contract} with $A$ replaced by $\widehat A$ gives that
for any $\alpha\in(0,\widehat\mu/\widehat L]$,
\begin{equation}
\label{eq:I-aA-contract}
  \|I-\alpha\widehat A\|_2^2
  \le
  1-\alpha\widehat\mu.
\end{equation}
Let $C_{\tref{lem:sensitivity-Gk},2}\doteq \sqrt{1-\alpha\widehat\mu}$, so $\|I-\alpha\widehat A\|_2\le C_{\tref{lem:sensitivity-Gk},2}<1$.
Combining with \eqref{eq:prod-bound} yields that for all $m\ge 0$,
\begin{equation}
\label{eq:prod-geometric}
  \bigl\|C^{-\frac12}(I-\alpha\widehat A)^{m}C^{\frac12}\bigr\|_2
  \le
  C_{\tref{lem:sensitivity-Gk},1}C_{\tref{lem:sensitivity-Gk},2}^{m}.
\end{equation}


For $E_2$, we bound $\|\nabla_\theta F(w_i) C^{\frac12}\|_2$ uniformly in $i$.
Using $\|\nabla_\theta F(w_i) C^{\frac12}\|_2=\|C^{\frac12}\nabla_\theta F(w_i)^\top\|_2$ and the definition of $\nabla_\theta F(w_i)$,
for any unit direction $u\in\R[p]$,
\begin{equation}
\label{eq:nablaDrel}
    D_\theta F(w_i)[u]
= (\nabla_\theta F(w_i))^\top u,
\quad
\|C^{\frac12}\nabla_\theta F(w_i)^\top\|_2
=
\sup_{\|u\|_2=1}\|C^{\frac12}D_\theta F(w_i)[u]\|_2.
\end{equation}
Recall $\mathcal F_\theta$,
under our block-sparse parameterization $\theta=(P,Q)\in \eqref{eq:Theta-CoT}$, we have for any $w_i$,
\begin{equation}
\label{eq:Ftheta-closed-again}
  \mathcal F_\theta(w_i)
  =
  w_i + \frac{1}{n}SPZ_i\Bigl(Z_i^\top Q{Z_i}_{[:,-1]}\Bigr),
\end{equation}
where $S\in\R[d\times(3d+1)]$ is the fixed selector matrix.
Differentiate \eqref{eq:Ftheta-closed-again} with respect to $\theta$ and evaluate at $\theta=\theta_* = (P_*, Q_*)$.
By the product rule and Lemma~\tref{lem:pqderiv},
for any $\|u\|_2=1$,
\begin{align}
\label{eq:DT-direction-again}
  D_\theta F(w_i)[u]
  &=
  \frac{1}{n}\,S\Bigl(
    \Delta P(u)Z_iZ_i^\top Q_*\,{Z_i}_{[:,-1]}
    +
    P_*Z_i\bigl(Z_i^\top \Delta Q(u)\,{Z_i}_{[:,-1]}\bigr)
  \Bigr),
\end{align}
where $\Delta P(u)$ and $\Delta Q(u)$ is defined in \eqref{eq:deltapq}.
Denote $Z_i=[Z_{\mathrm{ctx}},Z_{\mathrm{buf}}]$, where $Z_{\mathrm{ctx}}\in\R[(3d+1)\times n]$ is the submatrix consisting of the first $n$ context columns.
By Lemma~\tref{lem:pqderiv} and the block structure in \eqref{eq:para_construct},
the left multiplication of $P_*$ and $\Delta P(u)$ only acts on $Z_{\mathrm{ctx}}$, hence
\begin{equation}
\label{eq:PZ-restrict}
    P_*Z_i = P_*Z_{\mathrm{ctx}},
    \quad
    \Delta P(u)Z_i = \Delta P(u)Z_{\mathrm{ctx}}.
\end{equation}
Applying \eqref{eq:PZ-restrict} to \eqref{eq:DT-direction-again} yields
\begin{align}
\label{eq:DT-direction-restrict}
  D_\theta F(w_i)[u]
  &=
  \frac{1}{n}\,S\Bigl(
    \Delta P(u)Z_{\mathrm{ctx}}Z_{\mathrm{ctx}}^\top Q_*{Z_i}_{[:,-1]}
    +
    P_*Z_{\mathrm{ctx}}\bigl(Z_{\mathrm{ctx}}^\top \Delta Q(u){Z_i}_{[:,-1]}\bigr)
  \Bigr).
\end{align}
Denote $a_i \doteq Z_{\mathrm{ctx}}^\top Q_*{Z_i}_{[:,-1]}$.
Since $S$ is a selector matrix, $\|S\|_2=1$.
Taking norms in \eqref{eq:DT-direction-restrict} and using Lemma~\tref{lem:pqderiv} and Lemma~\tref{lem:pqstarbound} gives
\begin{align}
\label{eq:Tk-op-bound-1}
  \norm{D_\theta F(w_i)[u]}_2
  &\le
  \frac{1}{n}\Bigl(
    \norm{Z_{\mathrm{ctx}} a_i}_2
    +
    \norm{P_*}_2\norm{Z_{\mathrm{ctx}}}_2^2
    \norm{{Z_i}_{[:,-1]}}_2
  \Bigr) \notag\\
  &\le
  \frac{1}{n}\Bigl(
    \norm{Z_{\mathrm{ctx}} a_i}_2
    +
    \alpha\norm{Z_{\mathrm{ctx}}}_2^2\norm{{Z_i}_{[:,-1]}}_2
  \Bigr).
\end{align}
Therefore by \eqref{eq:nablaDrel},
\begin{align}
\label{eq:E2-bound-start}
  \|\nabla_\theta F(w_i) C^{\frac12}\|_2
  &=
  \|C^{\frac12}\nabla_\theta F(w_i)^\top\|_2\notag\\
  &\le
  C_{\tref{lem:sensitivity-Gk},3}
  \sup_{\|u\|_2=1}\|D_\theta F(w_i)[u]\|_2 \notag\\
  &\le
  \frac{C_{\tref{lem:sensitivity-Gk},3}}{n}\Bigl(
    \|Z_{\mathrm{ctx}} a_i\|_2
    +
    \alpha\,\|Z_{\mathrm{ctx}}\|_2^2\,\|{Z_i}_{[:,-1]}\|_2
  \Bigr),
\end{align}
where $C_{\tref{lem:sensitivity-Gk},3}\doteq \sqrt{\lambda_{\max}(C)}$.
We next bound the three factors one by one.

\noindent\textbf{Bound of $\|Z_{\mathrm{ctx}}\|_2$.}
By \eqref{eq:zt_def}, each context column satisfies
\[
  \|{Z_{\mathrm{ctx}}}_{[:,j]}\|_2
  \le
  \sqrt{(1+\gamma^2)C_x^2 + C_R^2}
  \doteq
  C_{\tref{lem:sensitivity-Gk},4}.
\]
Hence
\begin{equation}
\label{eq:Zctx-bound-again}
  \|Z_{\mathrm{ctx}}\|_2
  \le
  \|Z_{\mathrm{ctx}}\|_F
  \le
  \sqrt{n}\,C_{\tref{lem:sensitivity-Gk},4}.
\end{equation}

\noindent\textbf{Bound of $\|{Z_i}_{[:,-1]}\|_2$.}
By \eqref{eq:zt_def} and Lemma~\tref{lem:wbound},
\begin{equation}
\label{eq:zwk-bound-again}
  \|{Z_i}_{[:,-1]}\|_2 = \sqrt{1+\|w_i\|_2^2}
  \le
  \sqrt{1+C_{\tref{lem:wbound}}^2}
  \doteq
  C_{\tref{lem:sensitivity-Gk},5}.
\end{equation}

\noindent\textbf{Bound of $\|Z_{\mathrm{ctx}} a_i\|_2$.}
By the construction of $Q_*$ in \eqref{eq:para_construct}, for each $j$,
\[
  |(a_i)_j|
  \le
  C_R + (1+\gamma)C_x\|w_i\|_2
  \le
  C_R + (1+\gamma)C_x C_{\tref{lem:wbound}}.
\]
Thus
\begin{equation}
\label{eq:akctx-bound}
  \|a_i\|_2
  \le
  \sqrt{n}\Bigl(C_R + (1+\gamma)C_x C_{\tref{lem:wbound}}\Bigr).
\end{equation}
Combining \eqref{eq:Zctx-bound-again} and \eqref{eq:akctx-bound},
\begin{align}
\label{eq:Zkak-bound}
  \|Z_{\mathrm{ctx}} a_i\|_2
  \le
  \|Z_{\mathrm{ctx}}\|_2\,\|a_i\|_2
  \le
  nC_{\tref{lem:sensitivity-Gk},4}\Bigl(C_R + (1+\gamma)C_x C_{\tref{lem:wbound}}\Bigr).
\end{align}
Let $C_{\tref{lem:sensitivity-Gk},6}\doteq C_{\tref{lem:sensitivity-Gk},4}\qty(C_R + (1+\gamma)C_x C_{\tref{lem:wbound}})$.
Plugging \eqref{eq:Zctx-bound-again}, \eqref{eq:zwk-bound-again}, and \eqref{eq:Zkak-bound} into \eqref{eq:E2-bound-start} yields
\begin{align}
\label{eq:E2-uniform}
\|\nabla_\theta F(w_i) C^{\frac12}\|_2
\le
\frac{C_{\tref{lem:sensitivity-Gk},3}}{n}
\Bigl(
C_{\tref{lem:sensitivity-Gk},6}n
+
\alpha n C_{\tref{lem:sensitivity-Gk},4}^2  C_{\tref{lem:sensitivity-Gk},5}
\Bigr)
=C_{\tref{lem:sensitivity-Gk},3}(C_{\tref{lem:sensitivity-Gk},6}+\alpha C_{\tref{lem:sensitivity-Gk},4}^2  C_{\tref{lem:sensitivity-Gk},5})
\doteq
C_{\tref{lem:sensitivity-Gk},7},
\,\, \forall i\ge 0.
\end{align}
Plugging \eqref{eq:prod-geometric} and \eqref{eq:E2-uniform} into \eqref{eq:CJk-tri}, we obtain
\[
\|G_k(\theta_*)C^{\frac12}\|_2
\le
\sum_{i=0}^{k-1}
C_{\tref{lem:sensitivity-Gk},1}C_{\tref{lem:sensitivity-Gk},2}^{k-1-i}\,C_{\tref{lem:sensitivity-Gk},7}
=
C_{\tref{lem:sensitivity-Gk},8}\sum_{j=0}^{k-1}C_{\tref{lem:sensitivity-Gk},2}^{j}
\le
\frac{C_{\tref{lem:sensitivity-Gk},8}}{1-C_{\tref{lem:sensitivity-Gk},2}},
\]
where $C_{\tref{lem:sensitivity-Gk},8}\doteq C_{\tref{lem:sensitivity-Gk},1}C_{\tref{lem:sensitivity-Gk},7}$.
Let $C_\tref{lem:sensitivity-Gk} \doteq \frac{C_{\tref{lem:sensitivity-Gk},8}}{1-C_{\tref{lem:sensitivity-Gk},2}}$ then completes the proof.
\end{proof}

\subsection{Proof of Theorem~\tref{thm:emergence}}
\label{proof:emergence}

\begin{proof}
Set $C_{\text{Thm}\tref{thm:emergence}} \doteq \max\bigl(C_{\text{Thm}\tref{thm:finite},1},\ C_{\tref{lem:empirical-mu}}\bigr)$. Then \eqref{eq:neff} implies \eqref{eq:neff-threshold-mu} for each $\tau_n^{(i)}$, and also satisfies the sample-size requirement of Theorem~\tref{thm:finite} with $\delta'=\delta/N_{\mathrm{train}}$.
On the event $\mathcal E_\delta(\mathcal D)$, 
with \eqref{eq:neff-threshold-mu},
Lemma~\tref{lem:residual-Lhat} and
Lemma~\tref{lem:sensitivity-Gk} hold simultaneously for all
$i\in[N_{\mathrm{train}}]$.
For each $i$, using the factorization of the update direction at $\theta=\theta_*$ with \eqref{eq:neu-Gk-res}, we have
\begin{equation}
\label{eq:thm-emergence-factorization}
\Delta_k(\theta_*;\tau_n^{(i)})
=
G_k(\theta_*;\tau_n^{(i)})
\bigl(\widehat b^{(i)}-\widehat A^{(i)} w_k^{(i)}\bigr).
\end{equation}
By
$\widehat{\mathcal L}^{(i)}(w)\doteq
\|C^{-1/2}(\widehat A^{(i)}w-\widehat b^{(i)})\|_2^2$
we have, for all $k\ge 1$,
\begin{align*}
\bigl\|\Delta_k(\theta_*;\tau_n^{(i)})\bigr\|_2
&\le
\bigl\|G_k(\theta_*;\tau_n^{(i)})C^{1/2}\bigr\|_2\,
\sqrt{\widehat{\mathcal L}^{(i)}\!\bigl(w_k^{(i)}\bigr)} \\
&\le
C_{\tref{lem:sensitivity-Gk}}^{(i)}\,
\sqrt{C_\tref{lem:residual-Lhat}^{(i)}}\,
(1-\alpha\widehat\mu_{\mathcal D})^{k/2}\,
\sqrt{\widehat{\mathcal L}^{(i)}\!\bigl(w_0\bigr)} .
\end{align*}
Let
$C_{\text{Thm}\tref{thm:emergence},1}(\mathcal D)
\doteq
\max_{i\in[N_{\mathrm{train}}]}
C_\tref{lem:sensitivity-Gk}^{(i)}
\sqrt{C_\tref{lem:residual-Lhat}^{(i)}}.$
Applying the triangle inequality \eqref{eq:dataset-triangle} yields
\[
J_k(\theta_*;\mathcal D)
\le
\frac{1}{N_{\mathrm{train}}}\sum_{i=0}^{N_{\mathrm{train}}-1}
\bigl\|\Delta_k(\theta_*;\tau_n^{(i)})\bigr\|_2
\le
C_{\text{Thm}\tref{thm:emergence},1}(\mathcal D)\,
(1-\alpha\widehat\mu_{\mathcal D})^{k/2}\,
\frac{1}{N_{\mathrm{train}}}\sum_{i=0}^{N_{\mathrm{train}}-1}
\sqrt{\widehat{\mathcal L}^{(i)}\!\bigl(w_0\bigr)}.
\]
On $\cE_\delta(\mathcal D)$, Lemma~\ref{lem:empirical-mu} applies to each 
$\tau_n^{(i)} \in \mathcal D$, giving $\widehat\mu^{(i)} \geq \mu/2$ and 
$\widehat L^{(i)} \leq 4L$. Hence
$\widehat\mu_D = \min_i \widehat\mu^{(i)} \geq \mu/2$ and 
$\widehat L_D = \max_i \widehat L^{(i)} \leq 4L$, 
so $\alpha \leq \mu/(8L)$ implies $\alpha \leq \widehat\mu_D/\widehat L_D$.
Moreover, since
$\widehat\mu_{\mathcal D}\le \sqrt{\widehat L_{\mathcal D}}$,
the condition
$\alpha\in(0,\widehat\mu_{\mathcal D}/\widehat L_{\mathcal D}]$
implies
$\alpha\widehat\mu_{\mathcal D}\in(0,1]$.
Hence $0\le 1-\alpha\widehat\mu_{\mathcal D}<1$.
Therefore, $(1-\alpha\widehat\mu_{\mathcal D})^{k/2}\to 0$ as $k\to\infty$, and the above bound implies
\[
\lim_{k\to\infty} J_k(\theta_*;\mathcal D)=0.
\]
This completes the proof.
\end{proof}

\subsection{Proof of Corollary~\tref{cor:bridge}}
\label{proof:bridge}
\begin{proof}
By \eqref{eq:eventD}, $\mathbb P(\mathcal E_\delta(\mathcal D)) \ge 1-\delta$, so it suffices to prove the bound on $\mathcal E_\delta(\mathcal D)$. By \eqref{eq:eventD}, $\mathcal E_\delta(\mathcal D) \subseteq \mathcal E_{\delta'}^{(i)}$ for all $i$.
Fix any $i\in[N_{\mathrm{train}}]$.
Since \eqref{eq:neff-threshold-mu} holds for $\tau_n^{(i)}$, so $\widehat\mu^{(i)}>0$ on $\mathcal E_\delta(\mathcal D)$.
With $\alpha\in\bigl(0, \mu/(8L)\bigr]$,
the corresponding iterate sequence $\{w_k^{(i)}\}$ converges to $\widehat w_*^{(i)}$.
Since $\fL(\cdot)$ is continuous, we have
\[
\fL(\widehat w_*^{(i)})=\lim_{k\to\infty}\fL(w_k^{(i)}).
\]
Applying \eqref{eq:finite-main} to the trajectory $\tau_n^{(i)}$ with failure probability $\delta'$
and taking $k\to\infty$,
\[
\fL(\widehat w_*^{(i)})
\ \le\
C_{\text{Thm}\tref{thm:finite},2}\,\varepsilon(\delta')^2,
\]
which is exactly \eqref{eq:bridge-floor}. Since $i$ was arbitrary, the bound holds for all
$i\in[N_{\mathrm{train}}]$ on $\mathcal E_\delta(\mathcal D)$.
\end{proof}

\section{Additional Experimental Details}
\label{sec:aux_experiment}

\paragraph{Multi-task TD pretraining (Algorithm~\ref{alg:cot-td}).}
Algorithm~\ref{alg:cot-td} summarizes our multi-task reinforcement pretraining loop.
Each training task samples an MRP and a feature map from $d_{\textnormal{task}}$, rolls out a trajectory, and performs semi-gradient TD updates while the model generates a CoT sequence of weight iterates $\{w_t\}$ from the structured prompt.
In practice, pretraining is stabilized by gradient clipping with max norm $1.0$ and a small weight decay $10^{-6}$ (Table~\ref{tab:hyperparameters}).
We monitor the mean squared error between the Transformer's value prediction and the batch TD solution once every 10 MRPs on the current training MRP.
We early-stop when this error falls below $\xi=10^{-3}$ for 30 consecutive checks, meaning it stays below the threshold for $300$ MRPs.
We use the hyperparameters listed in Table~\ref{tab:hyperparameters} for experimentation.

\paragraph{Emergent block-sparse structure and equivalent solutions.}
We observe a clear emergence of block-sparse structure in $(P,Q)$, as shown in Figure~\tref{fig:pq-combined}.
As in \citet{wang2025ictd}, this emergence admits two equivalent solutions related by a global sign flip:
$(P,Q)$ and $(-P,-Q)$ implement the same in-context update, and different random seeds may converge to different sign conventions.
When aggregating across seeds, we fix a canonical representative by choosing the sign such that $Q_{[2d,2d]}>0$.

\paragraph{Boyan's chain (Figure~\ref{fig:boyan-topology} and Algorithm~\ref{alg:boyan-chain}).}
Figure~\ref{fig:boyan-topology} illustrates the sparse transition topology of Boyan's chain used in our experiments.
Algorithm~\ref{alg:boyan-chain} gives the randomization procedure. We sample a reward vector $r$, an initial distribution $p_0$, and a transition kernel $p$ with two nonzero forward transitions per state and a reset from the last state. We also sample an i.i.d.\ uniform feature map $x(\cdot)$.


\paragraph{Compute resources.}
All experiments were run on a single NVIDIA H100 GPU. 
The complete set of reported experiments, including all 10 random seeds in Figure~\tref{fig:pq-combined}, takes approximately 10 minutes of wall-clock time. 
This corresponds to roughly 0.17 H100 GPU-hours for the reported experimental results.
The experiments use only synthetic Boyan-chain data, and the storage requirement is negligible.
Preliminary and debugging runs used comparable compute and did not
substantially exceed the reported experimental budget.


\begin{algorithm}[t]
\caption{\label{alg:cot-td} Multi-Task Temporal Difference Learning with Chain-of-Thought}
\begin{algorithmic}[1]
\STATE \textbf{Input:}
task distribution $d_{\textnormal{task}}$,
context length $n$,
MRP sample length $\tau$,
number of training tasks $k$,
learning rate $\alpha$,
discount factor $\gamma$,
transformer parameters $\theta \doteq (P,Q)$.

\FOR{$i \gets 1$ \textbf{to} $k$}
    \STATE Sample an MRP $(p_0,p,r)$ and a feature map $x(\cdot)$ from $d_{\textnormal{task}}$.
    \STATE Sample $(S_0,R_1,S_1,\dots,S_{\tau+1})$ from the MRP.
    \STATE Initialize $w_0 \gets \mathbf{0}\in\mathbb{R}^d$.
    \STATE $\displaystyle
    Z_0 \gets
    \mqty[
        x_0 & \cdots & x_{n-1} & \mathbf 0 \\
        \gamma x_1 & \cdots & \gamma x_n & \mathbf 0 \\
        R_1 & \cdots & R_n & 1 \\
        \mathbf 0 & \cdots & \mathbf 0 & w_0
    ]$
\FOR{$t = 0, \dots, \tau - n - 1$}
    \STATE $z_t \gets \fF(Z_t;P,Q)_{[:, -1]} \in \mathbb{R}^{3d+1}$ \qquad // output column
    \STATE $w_{t+1} \gets z_{t,[2d+1:]} \in \mathbb{R}^d$ \qquad // weight vectors
    \STATE $v_t \gets w_{t+1}^\top x_{t+n}$ \qquad // query state is $S_{t+n}$
    
    \STATE $\displaystyle
    Z_{t+1} \gets
    \mqty[
        x_{t+1} & \cdots & x_{t+n} & z_{0,[0:d]} & \cdots & z_{t,[0:d]} \\
        \gamma x_{t+2} & \cdots & \gamma x_{t+n+1} & z_{0,[d:2d]} & \cdots & z_{t,[d:2d]} \\
        R_{t+2} & \cdots & R_{t+n+1} & z_{0,[2d]} & \cdots & z_{t,[2d]} \\
        \mathbf 0 & \cdots & \mathbf 0 & z_{0,[2d+1:]} & \cdots & z_{t,[2d+1:]}
    ]$
    \IF{$t \ge 1$}
        \STATE $\theta \gets \theta + \alpha\big(R_{t+n} + \gamma v_t - v_{t-1}\big)\nabla_\theta v_{t-1}$
        \qquad // semi-gradient TD
    \ENDIF
\ENDFOR
\ENDFOR
\end{algorithmic}
\end{algorithm}

\begin{figure}[t]
    \centering
    \includegraphics[width=0.52\linewidth]{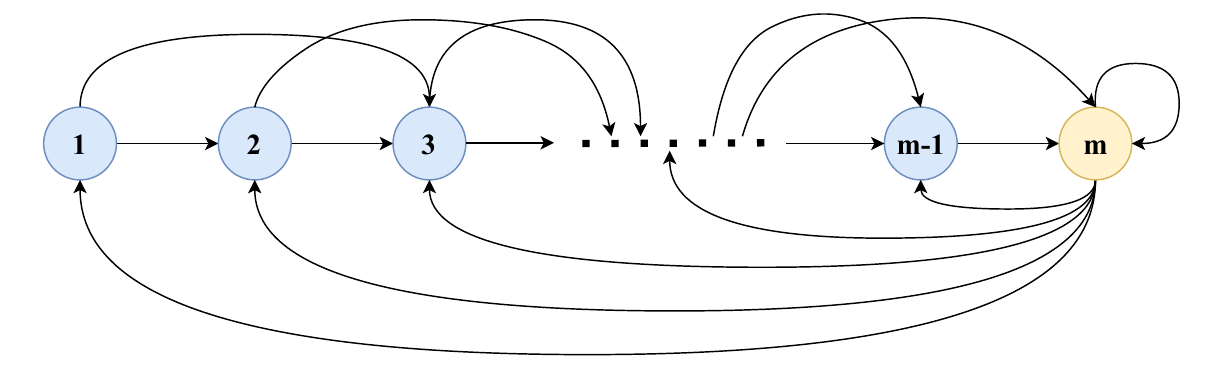}
    \caption{Boyan's chain topology with nonzero transitions. Adapted from \citet{wang2025ictd}.}
    \label{fig:boyan-topology}
\end{figure}

\begin{algorithm}[t]
\caption{Boyan Chain MRP and Feature Generation (Adapted from Algorithm~2 of \citet{wang2025ictd})}
\label{alg:boyan-chain}
\begin{algorithmic}[1]
\STATE \textbf{Input:} state space size $m=|\mathcal{S}|$, feature dimension $d$
\STATE $r \sim \mathrm{Uniform}\big[(-1,1)^m\big]$ \qquad // reward function 
\STATE $p_0 \sim \mathrm{Uniform}\qty[(0,1)^m]$ \qquad // initial distribution
\STATE $p_0 \gets p_0/\sum_s p_0(s)$
\STATE $p \gets 0_{m\times m}$ \qquad // transition function
\FOR{$i \gets 1$ to $m-2$}
    \STATE $\epsilon \sim \mathrm{Uniform}[(0,1)]$
    \STATE $p(i,i+1)\gets \epsilon$ 
    \STATE $p(i,i+2)\gets 1-\epsilon$
\ENDFOR
\STATE $p(m-1,m)\gets 1$
\STATE $z \sim \mathrm{Uniform}\big[(0,1)^m\big]$
\STATE $z \gets z/\sum_{s} z(s)$
\STATE $p(m,1:m)\gets z$
\FOR{each $s\in\mathcal{S}$}
    \STATE $x(s)\sim \mathrm{Uniform}\big[(-1,1)^d\big]$ \qquad // feature map
\ENDFOR
\STATE \textbf{Output:} MRP $(p_0,p,r)$ and feature map $x$
\end{algorithmic}
\end{algorithm}

\begin{table}[t]
  \centering
  \caption{Hyperparameters and training details}
  \label{tab:hyperparameters}
  \begin{tabular}{lc}
    \toprule
    \textbf{Hyperparameter} & \textbf{Value} \\
    \midrule
    \multicolumn{2}{l}{\textit{Optimization}} \\
    Optimizer & Adam \citep{kingma2014adam} \\
    Learning rate ($\alpha$) & $10^{-3}$ \\
    Weight decay & $10^{-6}$ \\
    Gradient clipping & $1.0$ \\
    \midrule
    \multicolumn{2}{l}{\textit{Model Architecture}} \\
    Number of layers ($L$) & $1$ \\
    Activation function & Identity (linear) \\
    Feature dimension ($d$) & $3$ \\
    Context length ($n$) & $30$ \\
    \midrule
    \multicolumn{2}{l}{\textit{Training}} \\
    Number of training MRPs ($k$) & $5000$ \\
    CoT iterations per parameter update & $16$ \\
    Updates per MRP & $10$ \\
    Training loss & squared TD error $\delta_t^2=(R_{t+n} + \gamma v_t - v_{t-1})^2$ \\
    \midrule
    \multicolumn{2}{l}{\textit{MRP Environment}} \\
    MRP type & Boyan chain \\
    Number of states ($m$) & $10$ \\
    Discount factor ($\gamma$) & $0.9$ \\
    \midrule
    \multicolumn{2}{l}{\textit{Early Stopping}} \\
    MSVE threshold ($\epsilon_{\text{stop}}$) & $10^{-3}$ \\
    Patience & $30$ \\
    \midrule
    \multicolumn{2}{l}{\textit{Other}} \\
    Random seeds & $0$--$9$ (10 runs) \\
    \bottomrule
  \end{tabular}
\end{table}


\end{document}